%% file: main.tex
\documentclass[10pt,twocolumn,letterpaper]{article}

\usepackage{cvpr}              

\usepackage{graphicx}
\usepackage{amsmath}
\usepackage{amssymb}
\usepackage{booktabs}
\usepackage{times}

\usepackage[pagebackref,breaklinks,colorlinks]{hyperref}

\usepackage{url}
\usepackage{enumitem}
\usepackage{booktabs}
\usepackage{amsthm}
\usepackage{amsmath}
\usepackage{amssymb}
\usepackage{multirow}
\usepackage{longtable}
\usepackage{setspace}
\usepackage{nicefrac}
\usepackage{xspace}
\usepackage{graphicx}
\usepackage{color}
\usepackage{colortbl}
\usepackage{setspace}
\usepackage{dsfont}
\usepackage{soul}
\usepackage{longtable}
\usepackage{caption}
\usepackage{subcaption}
\usepackage{algpseudocode, algorithm}
\usepackage[T1]{fontenc}
\usepackage{wrapfig}
\usepackage{enumitem}
\usepackage{multirow}
\usepackage{color}
\usepackage[dvipsnames]{xcolor}

\usepackage[accsupp]{axessibility}

\newcommand{\calD}{\mathcal{D}}
\newcommand{\calX}{\mathcal{X}}
\newcommand{\calY}{\mathcal{Y}}

\newcommand{\caltheta}{\Theta}

\newcommand{\bbR}{\mathbb{R}}
\newcommand{\bbP}{\mathbb{P}}
\newcommand{\bbE}{\mathbb{E}}

\newcommand{\bftheta}{\boldsymbol{\theta}}
\newcommand{\bfepsilon}{\boldsymbol{\epsilon}}

\newcommand{\bff}{\boldsymbol{f}}
\newcommand{\bfh}{\boldsymbol{h}}
\newcommand{\bfg}{\boldsymbol{g}}

\newcommand{\Loracle}{L^{\text{oracle}}}
\newcommand{\Lsam}{L^{\text{sam}}}
\newcommand{\Loverall}{L^{\text{overall}}}
\newcommand{\GAM}{R^{(1)}}
\newcommand{\sam}{R^{(0)}}
\DeclareMathOperator*{\argmax}{arg\,max}

\newcommand\blfootnote[1]{%
  \begingroup
  \renewcommand\thefootnote{}\footnote{#1}%
  \addtocounter{footnote}{-1}%
  \endgroup
}

\usepackage{thmtools,thm-restate}
\newtheorem{theorem}{Theorem}[section]

\newtheorem{lemma}[theorem]{Lemma}
\newtheorem{proposition}[theorem]{Proposition}
\theoremstyle{definition}
\newtheorem{definition}{Definition}[section]

\newtheorem*{remark}{Remark}

\newtheorem*{property*}{Property}

\newenvironment{itm}
{\begin{itemize}[leftmargin=*,noitemsep,topsep=0pt,parsep=0pt,partopsep=0pt]}
{\end{itemize}}

\usepackage{apptools}
\AtAppendix{\counterwithin{proposition}{section}}
\AtAppendix{\counterwithin{corollary}{section}}
\AtAppendix{\counterwithin{lemma}{section}}

\usepackage[capitalize]{cleveref}
\crefname{section}{Sec.}{Secs.}
\Crefname{section}{Section}{Sections}
\Crefname{table}{Table}{Tables}
\crefname{table}{Tab.}{Tabs.}

\def\cvprPaperID{912} 
\def\confName{CVPR}
\def\confYear{2023}

\makeatletter
\def\maketag@@@#1{\hbox{\m@th\normalfont\normalsize#1}}
\makeatother

\begin{document}

\title{Gradient Norm Aware Minimization Seeks First-Order Flatness and Improves Generalization}

\author{Xingxuan Zhang$^{\dag}$, Renzhe Xu$^{\dag}$, Han Yu, Hao Zou, Peng Cui*\\
Department of Computer Science, Tsinghua University\\
{\tt\small {xingxuanzhang@hotmail.com, xrz199721@gmail.com 
}} \\
{\tt\small {yuh21@mails.tsinghua.edu.cn,
zouh18@mails.tsinghua.edu.cn,
cuip@tsinghua.edu.cn}}}
\maketitle


\input{paras/abstract}

\input{paras/intro}

\input{paras/relatek}

\input{paras/preliminary}
\input{paras/method}
\input{paras/exp}

\input{paras/conclusion}
\input{paras/ack}

{\small
\bibliographystyle{ieee_fullname}
\bibliography{egbib}
}

\clearpage

\onecolumn
\appendix
\input{paras/appendix_method}
\input{paras/appendix_exp.tex}
\input{paras/appendix_accelerate}

\end{document}

%% file: paras/abstract.tex
\begin{abstract}
Recently, flat minima are proven to be effective for improving generalization and sharpness-aware minimization (SAM) achieves state-of-the-art performance. Yet the current definition of flatness discussed in SAM and its follow-ups are limited to the zeroth-order flatness (i.e., the worst-case loss within a perturbation radius). We show that the zeroth-order flatness can be insufficient to discriminate minima with low generalization error from those with high generalization error both when there is a single minimum or multiple minima within the given perturbation radius. Thus we present first-order flatness, a stronger measure of flatness focusing on the maximal gradient norm within a perturbation radius which bounds both the maximal eigenvalue of Hessian at local minima and the regularization function of SAM.  
We also present a novel training procedure named Gradient norm Aware Minimization (GAM) to seek minima with uniformly small curvature across all directions. Experimental results show that GAM improves the generalization of models trained with current optimizers such as SGD and AdamW on various datasets and networks. Furthermore, we show that GAM can help SAM find flatter minima and achieve better generalization. The code is available at \url{https://github.com/xxgege/GAM}.

\end{abstract}

\blfootnote{$\dag$Equal contribution, *Corresponding author}

%% file: paras/intro.tex
\section{Introduction}
\vspace{-5pt}
Current neural networks have achieved promising results in a wide range of fields \cite{ren2015faster,kipf2016semi,simonyan2014very,zhang2022towards2,young2018recent,zhang2021deep,zhou2021deepvit,zhang2019spatio}, yet they are typically heavily over-parameterized \cite{allen2019convergence,arora2018optimization}. 
Such heavy overparameterization leads to severe overfitting and poor generalization to unseen data when the model is learned simply with common loss functions (e.g., cross-entropy) \cite{izmailov2018averaging}.
Thus effective training algorithms are required to limit the negative effects of overfitting training data and find generalizable solutions.

\begin{figure}[t]
    \centering
    \begin{subfigure}[b]{1\linewidth}
        \centering
        \includegraphics[width=\linewidth]{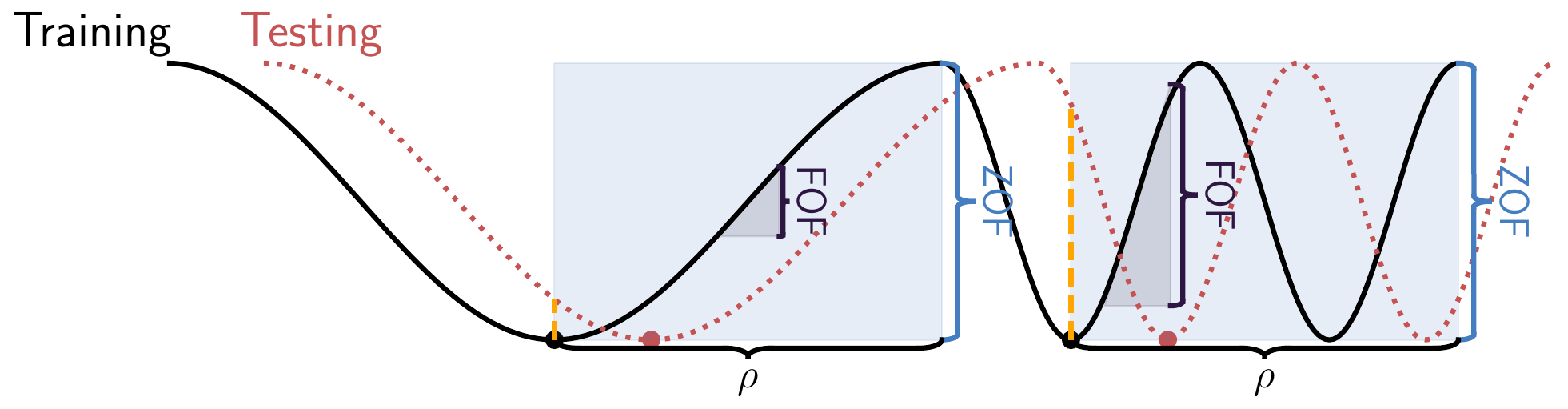}
        \caption{}
        \vspace{-2pt}
        \label{fig:sam-vs-GAM-1}
    \end{subfigure}
    \begin{subfigure}[b]{1\linewidth}
        \centering
        \includegraphics[width=\linewidth]{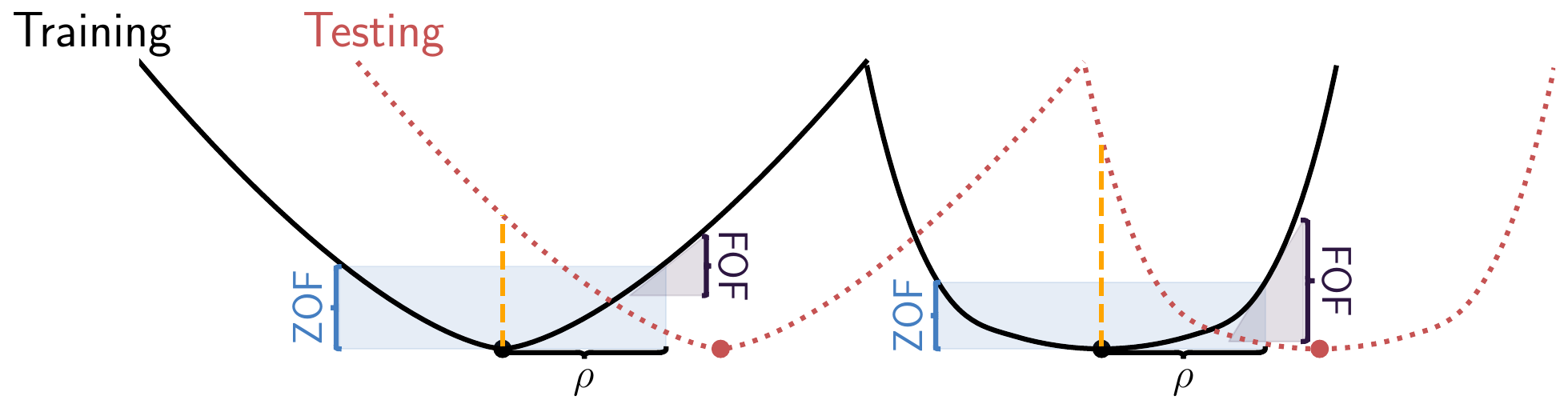}
        \caption{}
        \label{fig:sam-vs-GAM-2}
    \end{subfigure}
    \vspace{-20pt}
    \caption{The comparison of the zeroth-order flatness (ZOF) and first-order flatness (FOF). Given a perturbation radius $\rho$, ZOF can fail to indicate generalization error both when there are multiple minima (\ref{fig:sam-vs-GAM-1}) and a single minimum (\ref{fig:sam-vs-GAM-2}) in the radius while FOF remains discriminative. The height of blue rectangles in curly brackets is the value of ZOF and the height of gray triangles (which indicates the slope) is the value of FOF. In Figure \ref{fig:sam-vs-GAM-1}, when $\rho$ is large and enough to cover multiple minima, ZOF can not measure the fluctuation frequency while FOF prefers the flatter valley which has a smaller gradient norm. When $\rho$ is small and covers only a single minimum, the maximum loss in $\rho$ can be misleading as it can be misaligned with the uptrend of loss. As shown in Figure \ref{fig:sam-vs-GAM-2}, ZOF prefers the valley on the right, which has a larger generalization error (the orange dotted line), while FOF prefers the left one.}
    \label{fig:intuition}
    \vspace{-15pt}
\end{figure}

Many studies try to improve model generalization by modifying the training procedure, such as batch normalization \cite{ioffe2015batch}, dropout \cite{hinton2012improving}, and data augmentation \cite{zhang2018mixup,yun2019cutmix,cubuk2020randaugment}. Especially, some works discuss the connection between the geometry of the loss landscape and generalization \cite{izmailov2018averaging,he2019asymmetric,foret2021sharpness}. 
A branch of effective approaches, sharpness-Aware Minimization (SAM) \cite{foret2021sharpness} and its variants \cite{du2021efficient,liu2022towards, du2022sharpness,mi2022make,zhong2022improving,kim2022sharpness}, minimizes the worst-case loss within a perturbation radius, which we call zeroth-order flatness. It is proven that optimizing the zeroth-order flatness leads to lower generalization error and achieves state-of-the-art performance on various image classification tasks \cite{foret2021sharpness,zhuang2022surrogate,kwon2021asam}.

Optimizing the worst case, however, relies on a reasonable choice of perturbation radius $\rho$. As a prefixed hyperparameter in SAM or a hyperparameter under parameter re-scaling in its variants, such as ASAM \cite{kwon2021asam}, $\rho$ can not always be a perfect choice in the whole training process. 
We show that the zeroth-order flatness may fail to indicate the generalization error with a given $\rho$. As in Figure \ref{fig:sam-vs-GAM-1}, when $\rho$ covers multiple minima, the zeroth-order flatness (SAM) can not measure the fluctuation frequency. When there is a single minimum within $\rho$, as in Figure \ref{fig:sam-vs-GAM-2} the observation radius is limited and the maximum loss in $\rho$ can be misaligned with the uptrend of loss. So zeroth-order flatness can be misleading and the knowledge of loss gradient is required for generalization error minimization.

To address this problem, we introduce first-order flatness, which controls the maximum gradient norm in the neighborhood of minima. 
We show that the first-order flatness is stronger than the zeroth-order flatness as the loss intensity of the loss fluctuation can be bounded by the maximum gradient. When the perturbation radius covers multiple minima, which we show is quite common in practice, the first-order flatness discriminates more drastic jitters from real flat valleys, as in Figure \ref{fig:sam-vs-GAM-1}. When the perturbation radius is small and covers only one minimum, the first-order flatness demonstrates the trend of loss gradient and can help indicate generalization error. We further show that the first-order flatness directly controls the maximal eigenvalue of Hessian of the training loss, which is a proper sharpness/flatness measure indicating the loss uptrend under an adversarial perturbation to the weights  \cite{DBLP:conf/iclr/KeskarMNST17,jiang2019fantastic,kaur2022maximum}. 

To optimize the first-order flatness in deep model training, we propose Gradient norm Aware Minimization (GAM), which approximates the maximum gradient norm with stochastic gradient ascent and Hessian-vector products to avoid the materialization of the Hessian matrix. 

We summarize our contributions as follows.
\begin{itm}
    \item We present first-order flatness, which measures the largest gradient norm in the neighborhood of minima. We show that the first-order flatness is stronger than current zeroth-order flatness and it controls the maximum eigenvalue of Hessian.
    \item We propose a novel training procedure, GAM, to simultaneously optimize prediction loss and first-order flatness. We analyze the generalization error and the convergence of GAM.  
    \item We empirically show that GAM considerably improves model generalization when combined with current optimizers such as SGD and AdamW across a wide range of datasets and networks. We show that GAM further improves the generalization of models trained with SAM.
    \item We empirically validate that GAM indeed finds flatter optima with lower Hessian spectra.
\end{itm}

%% file: paras/relatek.tex
\section{Related Works}
\paragraph{Optimizer}
Some studies \cite{xie2022adaptive,foret2021sharpness} have demonstrated that current optimization approaches, such as SGD \cite{nesterov1983method}, Adam \cite{DBLP:journals/corr/KingmaB14}, AdamW \cite{loshchilov2017decoupled} and others \cite{duchi2011adaptive,DBLP:conf/iclr/LiuJHCLG020} affect generalization.
Some previous literature finds that Adam is more vulnerable to sharp minima than SGD \cite{wilson2017marginal}, which results in worse generalization ability \cite{xie2022power, hardt2016train, hochreiter1994simplifying}.
Some following works \cite{luo2019adaptive,chen2018closing,xie2022adaptive,zaheer2018adaptive} propose generalizable optimizers to address this problem. 
However, it can be a trade-off between generalization ability and convergence speed \cite{DBLP:conf/iclr/KeskarMNST17,xie2022adaptive,DBLP:conf/iclr/LiuJHCLG020,zaheer2018adaptive,duchi2011adaptive}. 
Different tasks and network architectures may agree with different optimizers (e.g., SGD is often chosen for ResNet \cite{he2016deep} while AdamW \cite{loshchilov2017decoupled} for ViTs \cite{dosovitskiy2020image}). Thus selecting a proper optimizer is critical while the understanding of its relationship to model generalization remains nascent \cite{foret2021sharpness}.

\paragraph{Flat Minima and Generalization}
Many recent works show that flatter minima lead to better generalization \cite{DBLP:conf/iclr/KeskarMNST17,DBLP:conf/iclr/KeskarMNST17,zhuang2022surrogate,jia2020information,petzka2021relative}.
Recently, \cite{kaur2022maximum} thoroughly reviews the literature related to generalization and sharpness of minima. It highlights the role of maximum Hessian eigenvalue in deciding the sharpness of minima \cite{DBLP:conf/iclr/KeskarMNST17, wen2019empirical}. And there also have been several simple strategies to achieve a smaller maximum Hessian eigenvalue, such as choosing a large learning rate \cite{lewkowycz2020large, DBLP:conf/iclr/CohenKLKT21, DBLP:conf/iclr/JastrzebskiSFAT20} and smaller batch size \cite{DBLP:conf/iclr/SmithL18, lewkowycz2020large, jastrzkebski2017three}. 
Sharpness-Aware Minimization (SAM) \cite{foret2021sharpness} and its variants \cite{zhuang2022surrogate, kwon2021asam, du2021efficient,liu2022towards, du2022sharpness,mi2022make,zhong2022improving,kim2022sharpness} are representative training algorithm to seek flat minima for better generalization. However, their definition of flatness is limited to zeroth-order flatness. In this paper, we present first-order flatness, a stronger flatness measure to learn better generalization. 
It is shown that discrete steps of gradient descent regularize deep models implicitly by penalizing the gradient descent trajectories with large loss gradients and this implicit regularization helps to find ﬂat minima \cite{barrett2020implicit}. 
\cite{zhao2022penalizing} proposes to directly control the gradient norm. They focus on the gradient norm at each training step, while we propose to penalize the maximum gradient norm in the neighborhood of minima and show the connection between our regularizer and the largest eigenvalue of Hessian and generalization error.

%% file: paras/preliminary.tex
\section{Preliminaries}
\paragraph{Notations} Let $\calX$ and $\calY$ be the sample space and label space, respectively. Let $\calD$ denote the training distribution on $\calX \times \calY$ and $S = \{(x_i, y_i)\}_{i=1}^n$ denote the training dataset with $n$ data-points drawn independently from $\calD$. Let $\bftheta \in \caltheta \subseteq \bbR^d$ denote the parameters of the model. In addition, we use $B(\bftheta, \rho)$ to denote the open ball of radius $\rho > 0$ centered at the point $\bftheta$ in the Euclidean space, \textit{i.e.}, $B(\bftheta, \rho) = \{\bftheta': \|\bftheta - \bftheta'\| < \rho\}$\footnote{We use $\|\cdot\|$ to denote the L2 norm throughout the paper.}.

Let $\ell: \caltheta \times \calX \times \calY \rightarrow \mathbb{R}$ be the per-data-point loss function. Let $\hat{L}(\bftheta) = \sum_{i=1}^n \ell(\bftheta, x_i, y_i)$ and $L(\bftheta) = \bbE_{(x, y) \sim \calD}[\ell(\bftheta, x, y)]$ denote the empirical loss function and population-level loss function, respectively. We assume $\hat{L}(\bftheta)$ and $L(\bftheta)$ are twice differentiable throughout the paper. $\nabla L(\bftheta)$ and $\nabla^2 L(\bftheta)$ ($\nabla \hat{L}(\bftheta)$ and $\nabla^2 \hat{L}(\bftheta)$) are the derivative and Hessian matrix of the function $L(\cdot)$ ($\hat{L}(\cdot)$) at point $\bftheta$, respectively. Besides, for any $\bftheta \in \caltheta$, we use $\nabla \|\nabla \hat{L}(\bftheta)\|$ to represent the gradient of function $\|\nabla \hat{L}(\cdot)\|$ at point $\bftheta$. In addition, we use $\Loracle(\bftheta)$ to denote an oracle loss function and it can be chosen as empirical loss function $\hat{L}(\bftheta)$, $\hat{L}(\bftheta)$ with the weight decay regularization, and other common loss functions.

\subsection{Zeroth-Order Flatness}
The most popular mathematical definitions of flatness considers the maximal loss value within a raduis \cite{DBLP:conf/iclr/KeskarMNST17,foret2021sharpness}, which we call the zeroth-order flatness. We follow the loss function proposed in SAM: 
\begin{equation} \label{eq:sam-loss}
\small
    \Lsam(\bftheta) = \hat{L}(\bftheta) + \max_{\bftheta' \in B(\bftheta, \rho)} \left(\hat{L}(\bftheta') - \hat{L}(\bftheta)\right).
\end{equation}

The second term in the right-hand side of Equation \eqref{eq:sam-loss} can be considered as a measure of the zeroth-order flatness.

\begin{definition}[$\rho$-zeroth-order flatness] \label{defn:zeroth-order}
    For any $\rho > 0$, the $\rho$-zeroth-order flatness $\sam_{\rho}(\bftheta)$ of function $\hat{L}(\bftheta)$ at a point $\bftheta$ is defined as
    \begin{equation} \label{eq:sam}
    \small
        \sam_{\rho}(\bftheta) \triangleq \max_{\bftheta' \in B(\bftheta, \rho)} \left(\hat{L}(\bftheta') - \hat{L}(\bftheta)\right), \quad \forall \bftheta \in \caltheta.
    \end{equation}
    Here $\rho$ is the perturbation radius that controls the magnitude of the neighborhood.
\end{definition}

Intuitively, we name the term zeroth-order flatness because it measures the gap between the maximum loss value and the current point. As a measure of accumulation of gradients, zeroth-order flatness can be insufficient to indicate the generalization loss as shown in Section \ref{sect:GAM-vs-sam}.
In this paper, we propose a novel first-order flatness measure and compare these two flatness notions in Section \ref{sect:GAM-vs-sam}.

%% file: paras/method.tex
\section{First-order Flatness and Optimization} \label{sect:method}
In this section, we introduce the first-order flatness and the corresponding minimizer for optimization. In Section \ref{sect:method-motivation}, we formulate the first-order flatness and show its connection with the maximal eigenvalue of the Hessian.
Afterward, we discuss the relationship between the zeroth-order and first-order flatness in Section \ref{sect:GAM-vs-sam}. In Section \ref{sect:method-GAM+sgd}, we present the optimization framework based on the first-order flatness as shown in Algorithm \ref{alg:GAM}. We further provide a generalization bound with respect to the empirical loss, the first-order flatness, and high order terms, indicating that optimizing the first-order flatness improves generalization abilities. We then prove the convergence of the algorithm.

\subsection{First-order Flatness} \label{sect:method-motivation}
We first introduce the formulation of the first-order flatness, which measures the maximal gradient norm in the neighbourhood of a point $\bftheta \in \caltheta$.
\begin{definition}[$\rho$-first-order flatness] \label{defn:first-order}
    For any $\rho > 0$, the $\rho$-first-order flatness $\GAM_{\rho}(\bftheta)$ of function $\hat{L}(\bftheta)$ at a point $\bftheta$ is defined as
    \begin{equation} \label{eq:GAM}
        \GAM_{\rho}(\bftheta) \triangleq \rho \cdot \max_{\bftheta' \in B(\bftheta, \rho)} \left\|\nabla \hat{L}(\bftheta')\right\|, \quad \forall \bftheta \in \caltheta.
    \end{equation}
    Here $\rho$ is the perturbation radius that controls the magnitude of the neighbourhood.
\end{definition}

Intuitively, the first-order flatness entails that the loss function $\hat{L}(\bftheta)$ should not change drastically in the neighbourhood of $\bftheta$ so that the largest gradient norm of loss is constrained.

We then discuss the relationship between the first-order flatness and  the maximal eigenvalue of the Hessian matrix $\nabla^2 \hat{L}(\bftheta^*)$ (denoted as $\lambda_{\max} (\nabla^2 \hat{L}(\bftheta^*))$). $\lambda_{\max}$ is proven to be a proper measure of the curvature of minima \cite{DBLP:conf/iclr/KeskarMNST17,kaur2022maximum} and is closely related to generalization abilities \cite{jastrzkebski2017three,wen2019empirical,chen2021vision}. 
As another definition of flatness in related works \cite{lewkowycz2020large,chaudhari2019entropy}, $\lambda_{\max}$ is widely accepted yet hard to calculate. We show in the following lemma that given a radius $\rho$, the first-order flatness controls $\lambda_{\max}$, which reinforces the validity of the first-order flatness.

\begin{lemma} \label{lemma:eigen-radius}
    Let $\bftheta^*$ be a local minimum of $\hat{L}$. Suppose $\hat{L}$ can be second-order Taylor approximated in the neighbourhood $B(\bftheta^*, \rho)$\footnote{The second order Taylor approximation assumption is commonly adopted in optimization-related literature \cite{mandt2017stochastic,zhang2019algorithmic,xie2020diffusion,xie2022adaptive} to analyze the properties near critical points.}, \textit{i.e.}, $\forall \bftheta \in B(\bftheta^*, \rho)$, $\hat{L}(\bftheta) = \hat{L}(\bftheta^*) + (\bftheta - \bftheta^*)^\top \nabla^2 \hat{L}(\bftheta^*)(\bftheta - \bftheta^*) / 2$. Then
    \begin{equation}
    \small
        \lambda_{\max} \left(\nabla^2 \hat{L}(\bftheta^*)\right) = \frac{\GAM_{\rho}(\bftheta^*)}{\rho^2}.
    \end{equation}
\end{lemma}

Since the maximal eigenvalue of Hessian matrices is usually difficult to approximate and optimize directly \cite{yao2018hessian,yao2020pyhessian}, the first-order flatness becomes a proper surrogate of $\lambda_{\max}$. 

\subsection{Comparison with Zeroth-order Flatness} \label{sect:GAM-vs-sam}
We compare the first-order flatness with the zeroth-order flatness. We first show that $\sam_{\rho}(\bftheta)$ in Equation \eqref{eq:sam} is bounded by $\GAM_{\rho}(\bftheta)$ in Equation \eqref{eq:GAM}.

\begin{proposition} \label{prop:GAM-and-sam}
    For any $\bftheta \in \caltheta$, $\sam_{\rho}(\bftheta)$ is bounded by $\GAM_{\rho}(\bftheta)$, \textit{i.e.}, $\GAM_{\rho}(\bftheta) \ge \sam_{\rho}(\bftheta)$.
\end{proposition}

Thus a smaller $\GAM_{\rho}$ also leads to a smaller $\sam_{\rho}$, indicating that $\GAM_{\rho}$ is a stronger flatness measure than $\sam_{\rho}$. 
Proposition \ref{prop:GAM-and-sam} gives an explanation that the first-order flatness covers wider scenarios compared with the zeroth-order flatness.

We present scenarios where the zeroth-order flatness fails to indicate generalization error while the first-order flatness remains discriminative in Figure \ref{fig:intuition}. 
The gap between a local minimum and the largest loss in $\rho$ can be considered as an accumulation of gradients across the trajectory while the largest gradient norm measures the maximum ascent rate, which may indicate the trends of loss outside of $\rho$.

When $\rho$ is large, there probably exist several other local minima in the neighborhood $B(\bftheta^*, \rho)$ as shown in Figure \ref{fig:sam-vs-GAM-1}. This case is common in practice as shown in Section \ref{sect:density-minima}. In addition, when the number of local minimum in $B(\bftheta^*, \rho)$ becomes larger, $\bftheta^*$ is expected to become sharper since the valley of $\bftheta^*$ becomes narrower. However, the zeroth-order flatness $\sam_{\rho}$ only measures the maximal gap of the loss function in $B(\bftheta^*, \rho)$ and fails to distinguish the cases when the number of local minimums varies. By contrast, the maximal gradient norm in $B(\bftheta^*, \rho)$ increases when the number of local minima is larger, indicating that the first-order flatness can successfully characterize the sharpness in this case.

When $\rho$ only covers a single minimum, as shown in Figure \ref{fig:sam-vs-GAM-2}, the zeroth-order flatness in $\rho$ can be misleading since the observation radius is insufficient to measure the loss trend with the maximum loss. The first-order flatness can help to learn more about the loss trend.

From the perspective of flatness, the zeroth-order flatness focuses on the average gradient within a radius while the first-order flatness measures the maximum gradient. Intuitively, the combination of the zeroth-order and first-order captures a more comprehensive picture of the loss landscape. Furthermore, as discussed in the following Section \ref{sect:method-GAM+sgd}, minimizers for both flatness measures adopt the first-order approximation to calculate the maxima within a radius. This may be the reason that the combination of the two flatness measures achieves the best performance as shown in Section \ref{sect:exp}.

\subsection{Gradient Norm Aware Minimization} \label{sect:method-GAM+sgd}
In this subsection, we propose a novel Gradient norm Aware Minimization (GAM) framework to incorporate the first-order flatness $\GAM_{\rho}(\bftheta)$ into optimization procedures.

Specifically, suppose we could obtain an oracle loss function $\Loracle(\bftheta)$ and calculate its gradient $\nabla \Loracle(\bftheta)$. $\Loracle(\bftheta)$ can be chosen as the empirical loss function $\hat{L}(\bftheta)$ and the empirical loss function with other regularizations (such as the weight decay and the zeroth-order flatness as shown in Definition \ref{defn:zeroth-order}). 

\paragraph{Generalization analysis} We first derive a generalization bound \textit{w.r.t.} the first-order flatness in Proposition \ref{prop:main-bound}.

\begin{proposition} \label{prop:main-bound}
    Suppose the per-data-point loss function $\ell$ is differentiable and bounded by $M$. Fix $\rho > 0$ and $\bftheta \in \caltheta$. Then with probability at least $1 - \delta$ over training set $S$ generated from the distribution $\calD$,
    \begin{equation} \label{eq:main-bound}
        \small
        \begin{aligned}
            & \, \bbE_{\epsilon_i \sim N(0, \rho^2/(\sqrt{d} + \sqrt{\log n})^2)}[L(\bftheta + \bfepsilon)] \\
            \le & \, \hat{L}(\bftheta) + \GAM_{\rho}(\bftheta) +  \frac{M}{\sqrt{n}} \\
            + & \sqrt{\frac{\frac{1}{4} d \log \left(1+\frac{\|\bftheta\|^2\left(\sqrt{d} + \sqrt{\log n}\right)^2}{d \rho^2}\right)+\frac{1}{4}+\log \frac{n}{\delta}+2 \log (6 n+3 d)}{n-1}}.
        \end{aligned}
    \end{equation}
\end{proposition}

\begin{remark}
    The left-hand side of Equation \eqref{eq:main-bound} is close to the population-level loss function $L(\bftheta)$ since the numbers of samples $n$ and parameters $d$ are often large. As a result, ignoring high-order terms, the population-level loss $L(\bftheta)$ is bounded by the empirical loss $\hat{L}(\bftheta)$ and the first-order flatness $\GAM_{\rho}(\bftheta)$, which motivates us to use $\GAM_{\rho}(\bftheta)$ as a regularizer to help improve the generalization abilities of models.
\end{remark}

Inspired by Lemma \ref{lemma:eigen-radius} and Proposition \ref{prop:main-bound}, the overall loss function is given by
\begin{equation}
\small
    \Loverall(\bftheta) = \Loracle(\bftheta) + \alpha \GAM_{\rho}(\bftheta),
\end{equation}
where $\alpha$ is a hyperparameter that determines the strength of regularization.
The gradient of the loss function $\Loverall(\bftheta)$ is given by $\nabla \Loverall(\bftheta) = \nabla \Loracle(\bftheta) + \alpha \nabla \GAM_{\rho}(\bftheta)$.
Using similar techniques in \cite{foret2021sharpness}, GAM approximates $\nabla \GAM_{\rho}(\bftheta)$ by
\begin{equation} \label{eq:GAM-derivative}
\small
    \begin{aligned}
        & \nabla \GAM_{\rho}(\bftheta)\approx \rho \cdot \nabla \left\|\nabla \hat{L} (\bftheta^{\text{adv}})\right\|, \quad \bftheta^{\text{adv}} = \bftheta + \rho \cdot \frac{\bff}{\|\bff\|}, \\
        & \bff = \nabla \left\|\nabla \hat{L}(\bftheta)\right\|.
    \end{aligned}
\end{equation}
Details of the derivation of $\nabla \GAM_{\rho}(\bftheta)$ can be found in Appendix A. Notice that
\begin{equation} \label{eq:gradient-norm}
\small
    \forall \bftheta \in \caltheta, \quad \nabla \|\nabla \hat{L}(\bftheta)\| = \frac{\nabla^2 \hat{L}(\bftheta) \cdot \nabla \hat{L}(\bftheta)}{\|\nabla \hat{L}(\bftheta)\|}.
\end{equation}
As a result, Equation \eqref{eq:GAM-derivative} can be calculated efficiently by the Hessian vector product. The pseudocode of the whole optimization procedure is shown in Algorithm \ref{alg:GAM}.

\paragraph{Convergence analysis}
We further analyze the convergence properties of GAM. Firstly, we introduce the Lipschitz smoothness, which is common adopted in optimization-related literature \cite{allen2018neon2,xu2018first,zhuang2022surrogate}.

\begin{definition} \label{defn:loss-smooth}
    A function $J: \caltheta \rightarrow \bbR$ is $\gamma$-Lipschitz smooth if
    \begin{equation}
        \forall \bftheta_1, \bftheta_2 \in \caltheta, \quad \left\|\nabla J(\bftheta_1) - \nabla J(\bftheta_2)\right\| \le \gamma \|\bftheta_1 - \bftheta_2\|.
    \end{equation}
\end{definition}

With Definition \ref{defn:loss-smooth}, we could prove the convergence property of GAM as shown in Theorem \ref{thrm:convergence-main}.

\begin{theorem} \label{thrm:convergence-main}
    Suppose $\Loracle(\bftheta)$ is $\gamma_1$-Lipschitz smooth and $\hat{L}(\bftheta)$ is $\gamma_2$-Lipschitz smooth. Suppose $|\Loracle(\bftheta)|$ is bounded by $M$. For any timestamp $t \in \{0, 1, \dots, T\}$ and any $\bftheta \in \caltheta$, suppose we can obtain noisy and bounded observations $g_t^{\text{loss}}(\bftheta)$, $g_t^{\text{norm}}(\bftheta)$, and $\tilde{g}_t^{\text{loss}}(\bftheta)$ of $\nabla \hat{L}(\bftheta)$, $\nabla \|\nabla \hat{L}(\bftheta)\|$, and $\nabla \Loracle(\bftheta)$ such that
    \begin{equation}
        \small
        \begin{aligned}
            & \bbE[g_t^{\text{loss}}(\bftheta)] = \nabla \hat{L}(\bftheta), \|g_t^{\text{loss}}(\bftheta)\| \le G^{\text{loss}},  \|g_t^{\text{norm}}(\bftheta)\| \le G^{\text{norm}}, \\
            & \bbE[\tilde{g}_t^{\text{loss}}(\bftheta)] = \nabla \Loracle(\bftheta), \|\tilde{g}_t^{\text{loss}}(\bftheta)\| \le \tilde{G}^{\text{loss}}.
        \end{aligned}
    \end{equation}
    Then with learning rate $\eta_t = \eta_0 / \sqrt{t}$ and perturbation radius $\rho_t = \rho_0 / \sqrt{t}$, GAM could obtain
    \begin{equation}
        \frac{1}{T}\sum_{t=1}^{T} \bbE\left[\left\|\nabla \Loverall(\bftheta_t)\right\|^2\right] \le \frac{C_1 + C_2 \log T}{\sqrt{T}},
    \end{equation}
    for some constants $C_1$ and $C_2$ that only depend on $\gamma, G^{\text{loss}}, G^{\text{norm}}, \tilde{G}^{\text{loss}}, M, \eta_0, \rho_0$, and $\alpha$. Here $\nabla \Loverall(\bftheta_t) = \nabla \Loracle(\bftheta_t) + \alpha \nabla \GAM_{\rho}(\bftheta_t)$ and $ \nabla \GAM_{\rho}(\bftheta_t)$ is approximated in Equation \eqref{eq:GAM-derivative}.
\end{theorem}
\begin{remark}
    The assumptions in Theorem \ref{thrm:convergence-main} are common and standard when analyzing convergence of non-convex functions via SGD-based methods \cite{DBLP:journals/corr/KingmaB14,reddi2018convergence,zhuang2022surrogate}. In addition, the requirements on $\Loracle(\bftheta)$ (\textit{i.e.}, $\Loracle(\bftheta)$ is Lipschitz smooth and we can obtain unbiased and bounded observations of $\nabla \Loracle(\bftheta)$) are mild and common. For example, when the empirical loss function $\hat{L}(\bftheta)$ satisfies the constraints, it is easy to check that $\hat{L}(\bftheta)$ with the weight decay regularization also meets the requirements.
\end{remark}

\begin{algorithm}[t]
\caption{Gradient norm Aware Minimization (GAM)}
\label{alg:GAM}
\begin{algorithmic}[1]
    \State \textbf{Input:} Batch size $b$, Learning rate $\eta_t$, Perturbation radius $\rho_t$, Trade-off coefficient $\alpha$, Small constant $\xi$
    \State $t \leftarrow 0$, $\bftheta_0 \leftarrow$ initial parameters
    \While{$\bftheta_t$ not converged}
        \State Sample $W_t$ from the training data with $b$ instances
        \State $\bfh_t^{\text{loss}} \leftarrow \nabla \Loracle(\bftheta_t)$ \Comment{Calculate the oracle loss gradient $\nabla \Loracle(\bftheta_t)$}
        \State $\bff_t \leftarrow \nabla^2 \hat{L}_{W_t}(\bftheta_t) \cdot \frac{\nabla \hat{L}_{W_t}(\bftheta_{t})}{\left\|\nabla \hat{L}_{W_t}(\bftheta_{t})\right\| + \xi}$
        \State $\bftheta^{\text{adv}}_t \leftarrow \bftheta_{t} + \rho_t \cdot \frac{\bff_t}{\|\bff_t\| + \xi}$
        \State $\bfh^{\text{norm}}_t \leftarrow \rho_t \cdot \nabla^2 \hat{L}_{W_t}(\bftheta_t^{\text{adv}}) \cdot \frac{\nabla \hat{L}_{W_t}(\bftheta_t^{\text{adv}})}{\left\|\nabla \hat{L}_{W_t}(\bftheta_t^{\text{adv}})\right\| + \xi}$ \Comment{Calculate the norm gradient $\nabla \GAM_{\rho_t}(\bftheta_t)$}
        \State $\bftheta_{t+1} \leftarrow \bftheta_{t} - \eta_t(\bfh^{\text{loss}}_t + \alpha \bfh^{\text{norm}}_t)$
        \State $t \leftarrow t + 1$
    \EndWhile
    \State \Return{$\bftheta_t$}
\end{algorithmic}
\end{algorithm}

%% file: paras/exp.tex
\section{Experiments}
\label{sect:exp}

We empirically show that the case discussed in Section \ref{sect:GAM-vs-sam} is common in practice.
Then we evaluate GAM with random initialization on various state-of-the-art models and the transfer learning setting on various datasets. 
We show the Hessian spectra of GAM at convergence and discuss the computation overhead of GAM with the considerable improvement of model generalization.

\begin{figure}[t]
    \centering
    \includegraphics[width=0.60\linewidth]{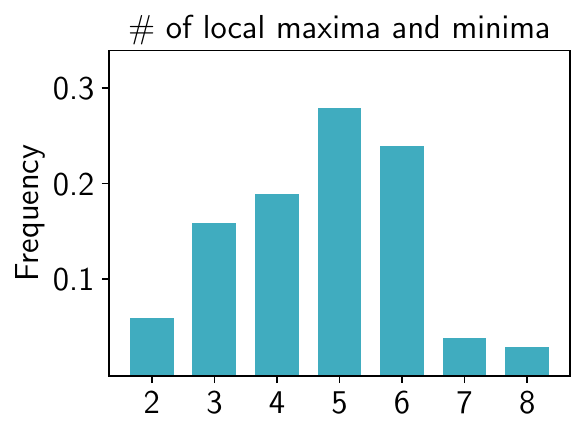}
    \vspace{-5px}
    \caption{The distribution of numbers of local minima and maxima within the perturbation radius $\rho$ after convergence.}
    \label{fig:local-minimum}
    \vspace{-15px}
\end{figure}

\subsection{The Density of Local Minima} \label{sect:density-minima}
To investigate the number of local minima within the perturbation radius, we train 3 ResNet-18 models with SAM on CIFAR-100 with proper hyperparameters for 200 epochs. The perturbation radius is set to 0.1 as suggested by \cite{foret2021sharpness}. We load the checkpoints at convergence for evaluation. We randomly generate 100 perturbation directions with the same size as the model weights for each model. For each direction, we repeatedly add a perturbation with the norm of 0.01 along the selected direction 10 times. We calculate the training loss after each addition and report the distribution of the number of local maxima and minima along each perturbation direction within the perturbation radius $\rho$ of 0.1. As shown in Figure \ref{fig:local-minimum}, we find more than 1 local minima within $\rho$ for most of the directions, indicating that the case is common in practice. As discussed in Section \ref{sect:GAM-vs-sam}, zeroth-order flatness fails to tell the sharpness caused by multiple minima while the first-order flatness measure increases as the sharpness grow.

\subsection{Training from Scratch}
\subsubsection{CIFAR-10 and CIFAR-100}
\label{sect:cifar}

\begin{table*}[t]
    \centering
    \small
    \caption{Results of GAM with state-of-the-art models on CIFAR-10 and CIFAR-100. The best results are highlighted in bold font.}
    \resizebox{\linewidth}{!}{
    \begin{tabular}{c|c|cc|cc|cc|cc}
        \toprule
        & & \multicolumn{4}{c|}{CIFAR-10} & \multicolumn{4}{c}{CIFAR-100} \\
        \midrule
        Model & Aug & SGD & SGD + GAM & SAM & SAM + GAM & SGD & SGD + GAM & SAM & SAM + GAM \\
        \midrule
        ResNet18 & Basic & 95.32$_{\pm 0.13}$ & \textbf{96.17}$_{\pm 0.21}$ & 96.10$_{\pm 0.20}$ & \textbf{96.75}$_{\pm 0.18}$ & 78.32$_{\pm 0.32}$ & \textbf{79.53}$_{\pm 0.30}$ & 79.27$_{\pm 0.16}$ & \textbf{80.45}$_{\pm 0.25}$ \\
        ResNet18 & Cutout & 95.99$_{\pm 0.13}$ & \textbf{96.46}$_{\pm 0.20}$ & 96.64$_{\pm 0.13}$ & \textbf{96.99}$_{\pm 0.23}$ & 78.73$_{\pm 0.13}$ & \textbf{79.89}$_{\pm 0.31}$ & 79.43$_{\pm 0.15}$ & \textbf{80.80}$_{\pm 0.14}$ \\
        ResNet18 & RA & 96.07$_{\pm 0.07}$ & \textbf{96.52}$_{\pm 0.09}$ & 96.64$_{\pm 0.17}$ & \textbf{97.06}$_{\pm 0.13}$ & 78.62$_{\pm 0.32}$ & \textbf{79.82}$_{\pm 0.24}$ & 79.71$_{\pm 0.15}$ & \textbf{80.97}$_{\pm 0.29}$  \\
        ResNet18 & AA & 96.13$_{\pm 0.05}$ & \textbf{96.71}$_{\pm 0.07}$ & 96.75$_{\pm 0.08}$ & \textbf{97.17}$_{\pm 0.08}$ & 78.88$_{\pm 0.15}$ & \textbf{80.56}$_{\pm 0.21}$ & 80.58$_{\pm 0.25}$ & \textbf{81.59}$_{\pm 0.24}$ \\
        \midrule
        ResNet101 & Basic & 96.35$_{\pm 0.08}$ & \textbf{96.98}$_{\pm 0.11}$ & 96.82$_{\pm 0.16}$ & \textbf{97.20}$_{\pm 0.15}$ & 80.47$_{\pm 0.13}$ & \textbf{82.21}$_{\pm 0.40}$ & 82.03$_{\pm 0.12}$ & \textbf{83.13}$_{\pm 0.07}$ \\
        ResNet101 & Cutout & 96.56$_{\pm 0.18}$ & \textbf{97.22}$_{\pm 0.05}$ & 97.07$_{\pm 0.08}$ & \textbf{97.36}$_{\pm 0.24}$ & 80.53$_{\pm 0.30}$ & \textbf{82.36}$_{\pm 0.24}$ & 81.60$_{\pm 0.35}$ & \textbf{83.40}$_{\pm 0.13}$ \\
        ResNet101 & RA & 96.68$_{\pm 0.25}$ & \textbf{97.33}$_{\pm 0.30}$ & 97.12$_{\pm 0.18}$ & \textbf{97.40}$_{\pm 0.23}$ & 80.60$_{\pm 0.28}$ & \textbf{82.40}$_{\pm 0.31}$ & 82.19$_{\pm 0.34}$ & \textbf{83.28}$_{\pm 0.20}$ \\
        ResNet101 & AA & 96.78$_{\pm 0.14}$ & \textbf{97.39}$_{\pm 0.18}$ & 97.18$_{\pm 0.11}$ & \textbf{97.42}$_{\pm 0.1}$ & 81.83$_{\pm 0.37}$ & \textbf{83.19}$_{\pm 0.15}$ & 82.44$_{\pm 0.47}$ & \textbf{83.94}$_{\pm 0.23}$ \\
        \midrule
        WRN28\_2 & Basic & 94.82$_{\pm 0.07}$ & \textbf{95.69}$_{\pm 0.13}$ & 95.47$_{\pm 0.08}$ & \textbf{95.85}$_{\pm 0.08}$ & 75.45$_{\pm 0.25}$ & \textbf{77.21}$_{\pm 0.31}$ & 77.04$_{\pm 0.18}$ & \textbf{77.69}$_{\pm 0.20}$ \\
        WRN28\_2 & Cutout & 95.70$_{\pm 0.20}$ & \textbf{96.41}$_{\pm 0.18}$ & 96.22$_{\pm 0.13}$ & \textbf{96.39}$_{\pm 0.22}$ & 76.80$_{\pm 0.45}$ & \textbf{78.58}$_{\pm 0.24}$ & 78.04$_{\pm 0.43}$ & \textbf{79.33}$_{\pm 0.12}$ \\
        WRN28\_2 & RA & 95.75$_{\pm 0.16}$ & \textbf{96.35}$_{\pm 0.13}$ & 96.22$_{\pm 0.08}$ & \textbf{96.49}$_{\pm 0.20}$ & 76.73$_{\pm 0.27}$ & \textbf{78.66}$_{\pm 0.03}$ & 77.88$_{\pm 0.29}$ & \textbf{78.96}$_{\pm 0.13}$ \\
        WRN28\_2 & AA & 95.44$_{\pm 0.06}$ & \textbf{95.98}$_{\pm 0.09}$ & 96.07$_{\pm 0.08}$ & \textbf{96.44}$_{\pm 0.09}$ & 77.35$_{\pm 0.02}$ & \textbf{79.05}$_{\pm 0.10}$ & 78.64$_{\pm 0.23}$ & \textbf{79.50}$_{\pm 0.21}$ \\
        \midrule
        WRN28\_10 & Basic & 95.73$_{\pm 0.10}$ & \textbf{96.61}$_{\pm 0.15}$ & 96.78$_{\pm 0.80}$ & \textbf{97.29}$_{\pm 0.11}$ & 81.40$_{\pm 0.13}$ & \textbf{83.45}$_{\pm 0.09}$ & 83.41$_{\pm 0.04}$ & \textbf{84.31}$_{\pm 0.06}$ \\
        WRN28\_10 & Cutout & 96.74$_{\pm 0.03}$ & \textbf{96.97}$_{\pm 0.05}$ & 97.35$_{\pm 0.16}$ & \textbf{97.56}$_{\pm 0.12}$ & 81.53$_{\pm 0.40}$ & \textbf{83.69}$_{\pm 0.08}$ & 82.38$_{\pm 0.15}$ & \textbf{84.43}$_{\pm 0.13}$ \\
        WRN28\_10 & RA & \textbf{97.14}$_{\pm 0.04}$ & 96.83$_{\pm 0.03}$ & \textbf{97.58}$_{\pm 0.07}$ & 97.49$_{\pm 0.03}$ & 81.65$_{\pm 0.18}$ & \textbf{83.84}$_{\pm 0.09}$ & 82.79$_{\pm 0.06}$ & \textbf{84.68}$_{\pm 0.13}$ \\
        WRN28\_10 & AA & 96.93$_{\pm 0.12}$ & \textbf{97.05}$_{\pm 0.04}$ & 97.48$_{\pm 0.06}$ & \textbf{97.67}$_{\pm 0.08}$ & 81.99$_{\pm 0.11}$ & \textbf{84.02}$_{\pm 0.18}$ & 83.84$_{\pm 0.30}$ & \textbf{84.81}$_{\pm 0.21}$ \\
        \midrule
        PyramidNet110 & Basic & 96.19$_{\pm 0.11}$ & \textbf{97.11}$_{\pm 0.14}$ & 97.26$_{\pm 0.05}$ & \textbf{97.51}$_{\pm 0.09}$ & 82.74$_{\pm 0.12}$ & \textbf{84.91}$_{\pm 0.09}$ & 85.01$_{\pm 0.09}$ & \textbf{85.25}$_{\pm 0.06}$ \\
        PyramidNet110 & Cutout & 96.82$_{\pm 0.09}$ & \textbf{97.32}$_{\pm 0.21}$ & 97.49$_{\pm 0.06}$ & \textbf{97.91}$_{\pm 0.14}$ & 83.31$_{\pm 0.21}$ & \textbf{85.20}$_{\pm 0.19}$ & 84.90$_{\pm 0.03}$ & \textbf{85.46}$_{\pm 0.10}$ \\
        PyramidNet110 & RA & 97.15$_{\pm 0.21}$ & \textbf{97.80}$_{\pm 0.22}$ & 97.60$_{\pm 0.09}$ & \textbf{98.01}$_{\pm 0.10}$ & 84.04$_{\pm 0.19}$ & \textbf{86.47}$_{\pm 0.14}$ & 85.33$_{\pm 0.27}$ & \textbf{85.64}$_{\pm 0.20}$ \\
        PyramidNet110 & AA & 97.11$_{\pm 0.01}$ & \textbf{97.85}$_{\pm 0.02}$ & 97.61$_{\pm 0.14}$ & \textbf{97.95}$_{\pm 0.10}$ & 84.48$_{\pm 0.03}$ & \textbf{85.92}$_{\pm 0.03}$ & 85.69$_{\pm 0.17}$ & \textbf{86.35}$_{\pm 0.18}$ \\
        \bottomrule
    \end{tabular}}
    \vspace{-10px}
    \label{tab:cifar2}
\end{table*}

We conduct experiments on CIFAR-10 and CIFAR-100 \cite{krizhevsky2009learning} with ResNets \cite{he2016deep}, WideResNet \cite{zagoruyko2016wide}, ResNeXt \cite{xie2017aggregated}, PyramidNet \cite{han2017deep} and Vision Transformers (ViTs) \cite{dosovitskiy2020image}. All the models are trained for 200 epochs from scratch. We evaluate GAM both with basic data augmentations (i.e., horizontal flip, padding by four pixels, and random crop) and advanced data augmentation including cutout regularization \cite{devries2017improved}, RandAugment \cite{cubuk2018autoaugment} and AutoAugment \cite{cubuk2020randaugment}. 

GAM has two hyperparameters, $\rho$ and $\alpha$. We conduct a grid search over $\{ 0.05, 0.1, 0.2, 0.5, 1.0, 2.0\}$ to tune $\rho$ and $\{0.1, 0.2, 0.5, 1.0, 2.0, 3.0, ..., 10.0 \}$ for $\alpha$ using 10\% of the training data as a validation set. The selection of hyperparameters is in Appendix C.5.

As a gradient regularizer, GAM can be integrated with current optimizers such as SGD and Adam \cite{DBLP:conf/iclr/KeskarMNST17}. We also show that GAM can be combined with sharpness-aware training procedures such as SAM. As shown in Section \ref{sect:GAM-vs-sam}, the GAM term bounds the regularization term in SAM. Yet the practical implementations of GAM and SAM rely on first-order Taylor expansion of different objective functions (GAM approximates the maximum gradient norm while SAM approximates the maximum loss). We empirically show that the combination of GAM and SAM outperforms both of them, indicating that they may strengthen each other with omitted items.

As shown in 
Table \ref{tab:cifar2}, GAM improves generalization for all models on CIFAR-10 and CIFAR-100. When combined with SGD, GAM achieves considerably higher test accuracy compared with SGD. Moreover, GAM further improves generalization when combined with SAM. For example, GAM improves SAM performance by 1.18\% and 1.10\%  on CIFAR-100 with ResNet-18 and ResNet-101, respectively, which are noticeable margins. Other experimental results are in Appendix C.1.

\subsubsection{ImageNet}
\begin{table*}[t]
    \centering
    \small
    \caption{Results of GAM with ResNet50 on ImageNet.}
    \vspace{-10px}
    \begin{tabular}{c|c|cc|cc}
        \toprule
         Model & Dataset & Base Opt & Base + GAM & SAM & SAM + GAM \\
        \midrule
        ResNet50 & Top-1 & 76.01$_{\pm 0.19}$ & \textbf{76.59}$_{\pm 0.15}$ & 76.47$_{\pm 0.11}$ & \textbf{76.86}$_{\pm 0.15}$  \\
        ResNet50 & Top-5 & 92.75$_{\pm 0.08}$ & \textbf{93.10}$_{\pm 0.08}$ & 93.07$_{\pm 0.05}$ & \textbf{93.22}$_{\pm 0.06}$  \\
        \midrule
        ResNet101 & Top-1 & 77.69$_{\pm 0.08}$ & \textbf{78.45}$_{\pm 0.10}$ & 78.35$_{\pm 0.12}$ & \textbf{78.70}$_{\pm 0.12}$ \\
        ResNet101 & Top-5 & 93.76$_{\pm 0.09}$ & \textbf{94.09}$_{\pm 0.12}$ & 94.02$_{\pm 0.06}$ & \textbf{94.15}$_{\pm 0.12}$ \\
        \midrule
        ViT-S/32 & Top-1 & 68.26$_{\pm 0.22}$ & \textbf{69.95}$_{\pm 0.16}$ & 69.73$_{\pm 0.05}$ & \textbf{70.15}$_{\pm 0.18}$ \\
        ViT-S/32 & Top-5 & 87.39$_{\pm 0.19}$ & \textbf{88.11}$_{\pm 0.26}$ & 87.91$_{\pm 0.30}$ & \textbf{88.23}$_{\pm 0.18}$ \\
        \midrule
        ViT-B/32 & Top-1 & 71.15$_{\pm 0.14}$ & \textbf{73.58}$_{\pm 0.06}$ & 73.10$_{\pm 0.18}$ & \textbf{73.70}$_{\pm 0.10}$ \\
        ViT-B/32 & Top-5 & 90.12$_{\pm 0.07}$ & \textbf{91.15}$_{\pm 0.19}$ & 91.03$_{\pm 0.06}$ & \textbf{91.50}$_{\pm 0.16}$ \\
        \bottomrule
    \end{tabular}
    \label{tab:imagenet}
\end{table*}

We use ResNet50, ResNet101 \cite{he2016deep}, ViT-S/32 and ViT-B/32 \cite{dosovitskiy2020image} for evaluations on ImageNet \cite{russakovsky2015imagenet} to evaluate GAM on large scale data. 
For ResNet, we use SGD with momentum= 0.9 as the base optimizer for both GAM and SAM. For ViT, we use the AdamW optimizer with $\beta_1 = 0.9$, $\beta_2 = 0.999$. 
We train ResNets for 90 epochs and ViTs for 300 epochs following \cite{dosovitskiy2020image}. 
We set the batch size to 256, learning rate to 0.1, and weight decay to 0.0001. The learning rate is decayed using a cosine schedule. 

As shown in Table \ref{tab:imagenet}, GAM consistently improves SGD performance on ImageNet for both ResNets and ViTs. GAM also further improves the model generalization compared with SAM. The combination of GAM and SAM outperforms both SGD and SAM by a noticeable margin.

\subsection{Transfer Learning}

\begin{table*}[t]
    \centering
    \small
    \caption{Results of GAM for finetuning EfficientNet-b0 and Swin Transformers on various datasets.}
    \vspace{-5px}
    \resizebox{0.9\linewidth}{!}{
    \begin{tabular}{c|cc|cc|cc|cc}
        \toprule
        & \multicolumn{4}{c|}{EfficientNet-b0} & \multicolumn{4}{c}{Swin-t} \\
        \midrule
         Dataset & SGD & SGD + GAM & SAM & SAM + GAM & AdamW & AdamW + GAM & SAM & SAM + GAM \\
        \midrule
        Stanford Cars & 82.14 & \textbf{83.50} & 83.21 & \textbf{83.98} & 83.50 & \textbf{84.90} & 83.55 & \textbf{85.29} \\ 
        CIFAR-10 & 86.26 & \textbf{87.37} & 86.95 & \textbf{87.97} & 91.32 & \textbf{92.06} & 91.77 & \textbf{92.55} \\ 
        CIFAR-100 & 63.75 & \textbf{64.85} & 64.29 & \textbf{65.03} & 72.88 & \textbf{73.78} & 73.99 & \textbf{74.30} \\ 
        Oxford\_IIIT\_Pets & 91.03 & \textbf{91.80} & 91.65 & \textbf{91.96} & 93.49 & \textbf{93.87} & 93.59 & \textbf{94.03} \\
        Food101 & 82.54 & \textbf{82.69} & 82.57 & \textbf{83.01} & 86.38 & \textbf{86.89} & 86.64 & \textbf{87.03} \\
        \bottomrule
    \end{tabular}}
    \vspace{-10px}
    \label{tab:transfer}
\end{table*}

Transfer learning shows the generalization of models when trained on sufficient labeled data and finetuned on a small dataset \cite{zhuang2020comprehensive}. We show that GAM improves generalization on all datasets in this setting.

We consider Stanford Cars \cite{krause20133d}, CIFAR-10, CIFAR-100 \cite{krizhevsky2009learning}, Oxford\_IIIT\_Pets \cite{parkhi2012cats} and Food101 \cite{bossard2014food} for this setting. We apply SGD, SAM, and GAM to finetuning EfficientNet-b0 \cite{tan2019efficientnet} and Swin-Transformer-t \cite{liu2021swin} on these datasets. Both EfficientNet-b0 and Swin-Transformer-t are pretrained on ImageNet.  

We use ImageNet pretrained weights of EfficientNet-b0 and Swin-t except for the last linear layer for classification. Following previous works, we train for 40k steps since our batch size is 128. The initial learning rate is set to 2e-3 with cosine learning rate decay. Weight decay is set to 1e-5. We do not use any data augmentations for Stanford Cars, Oxford\_IIIT\_Pets and Food101. For CIFAR datasets, we employ the same data augmentations as previous experiments. 

As seen in Table \ref{tab:transfer}, GAM once again brings generalization improvement for SGD, AdamW, and SAM on both EfficientNet-b0 and Swin-t. For example, GAM improves AdamW by 1.2\% on Stanford Cars with Swin-t and 1.11\% on CIFAR-10 with EfficientNet-b0.

Moreover, we leave the experiments of robustness to label noise in Appendix C.2.

\subsection{Top Eigenvalues of Hessian and Hessian Trace} \label{sect:exp-eigenvalue}

\begin{figure*}[t]
    \centering
    \includegraphics[width=0.8\linewidth]{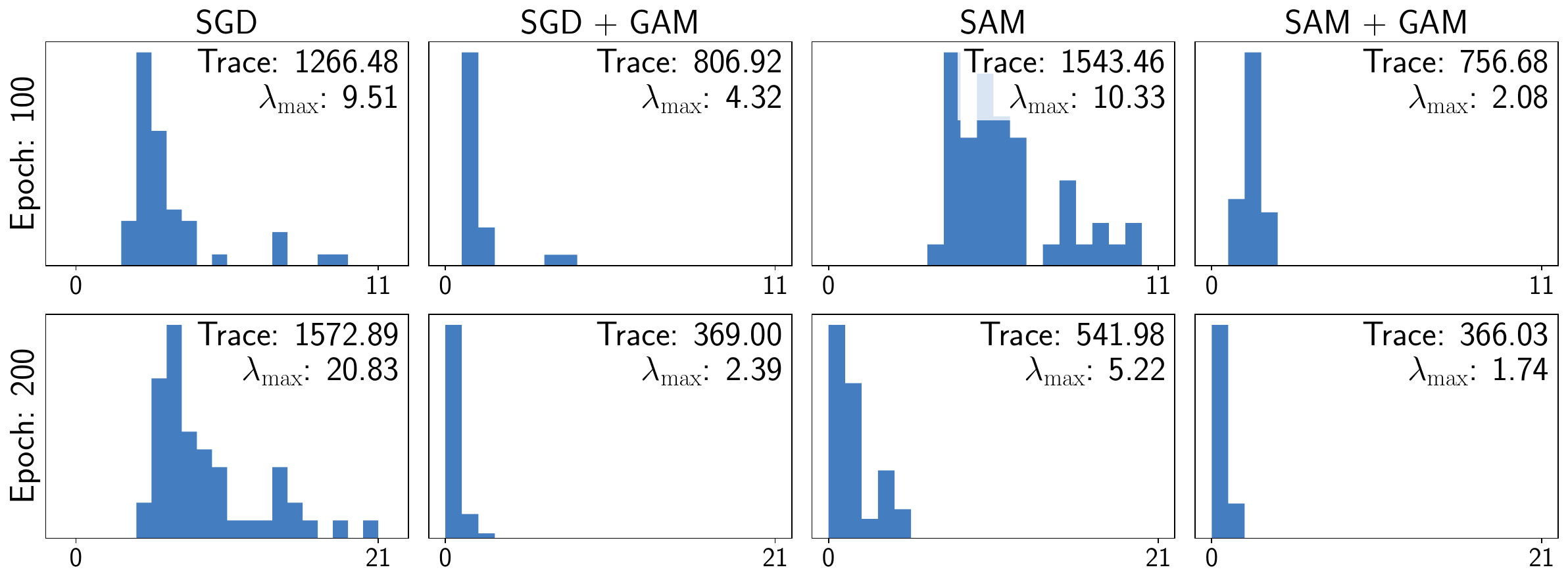}
    \vspace{-10pt}
    \caption{The distribution of top eigenvalues and the trace of Hessian at epoch 100 and 200 on CIFAR-100 with SGD, SGD + GAM, SAM, or SAM + GAM.}
    \label{fig:eigenvalue}
    \vspace{-15pt}
\end{figure*}

\begin{figure}[ht]
    \centering
    \includegraphics[width=\linewidth]{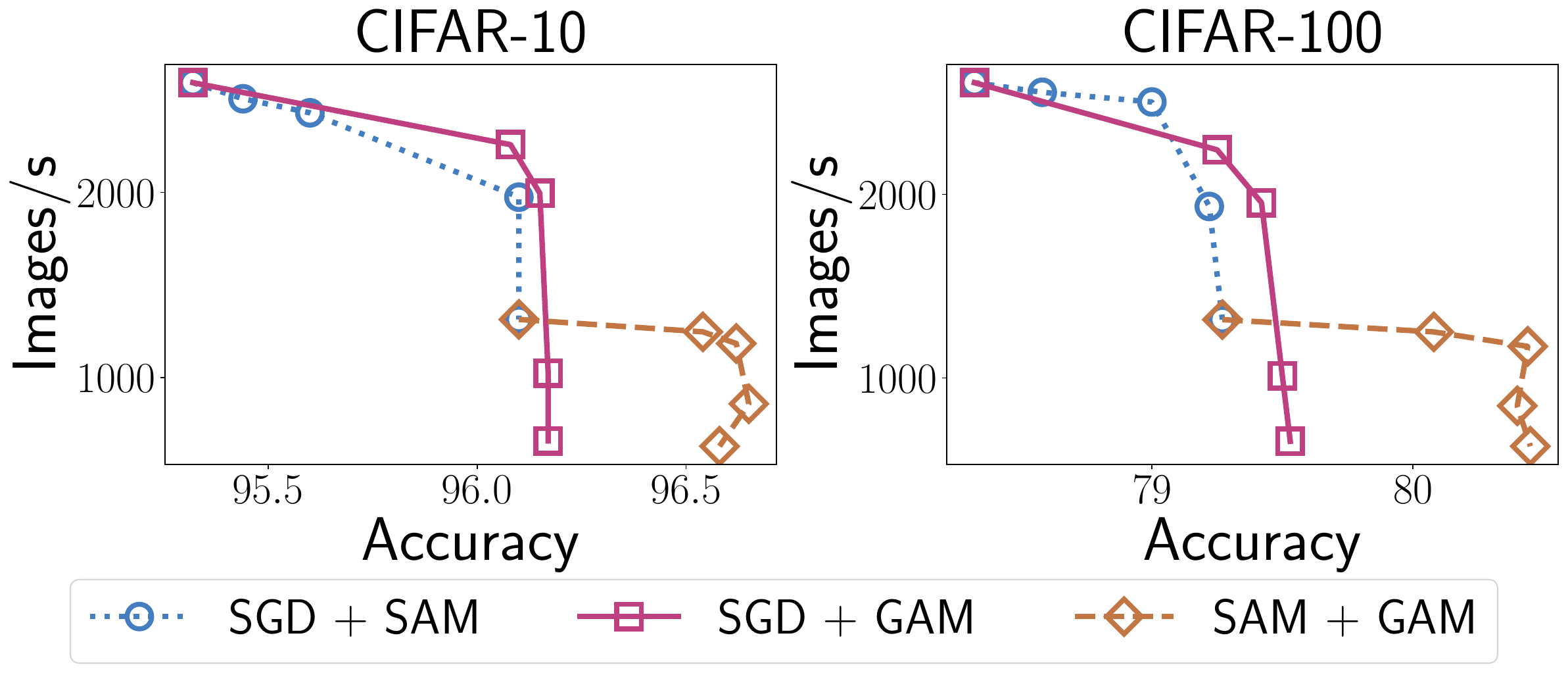}
    \caption{Accuracy and training speed of training with different ratios ([0, 0.05, 0.1, 0.5, 1] from upper left to lower right, see details in Appendix C.3) of iterations using GAM. Numbers in parentheses indicate the ratio of the training speed compared with the vanilla base optimizer SGD/SAM. }
    \label{fig:time}
    \vspace{-10pt}
\end{figure}

Lemma \ref{lemma:eigen-radius} shows that the GAM term can be an equivalent measure of the maximum eigenvalue of the Hessian, which is a well-known measure of flatness/sharpness. Thus optimizing the GAM term decreases the maximum eigenvalue of the Hessian and leads to flatter minima. To empirically validate that GAM finds optima with low curvature, we present the Hessian spectra of SGD, SAM, and GAM. We consider the maximum eigenvalue of Hessian and the Hessian trace, which measures the expected loss increase under random perturbations to the weights \cite{kaur2022maximum} as the measures of flatness. We empirically show that GAM significantly decreases both the maximum eigenvalue and the trace of Hessian during training compared with SGD and SAM, and thus finds flatter minima.

We compute the Hessian spectra of ResNet-18 trained on CIFAR-100 for 200 epochs with SGD, SAM, SGD + GAM, and SAM + GAM. We use power iteration \cite{yao2018hessian} to compute the top eigenvalues of Hessian and Hutchinson’s method \cite{avron2011randomized,bai1996some,yao2020pyhessian} to compute the Hessian trace. We report the histogram of the distribution of the top-50 Hessian eigenvalues for each method.

As shown in Figure \ref{fig:eigenvalue}, the model trained with SGD has a higher maximum Hessian eigenvalue and Hessian trace at convergence compared to the middle of training, indicating that optimizing directly with cross-entropy loss does not contribute to the lower Hessian spectra. In contrast, GAM leads to lower Hessian spectra and thus flatter minima. Moreover, GAM helps to reduce both top eigenvalues and the Hessian trace when combined with SAM, where Hessian spectra at convergence are lower than other methods. We show visualizations of landscapes of SGD, SAM, and GAM in Section \ref{app:vis}.

\subsection{Computation Overhead}

As discussed in Section \ref{sect:method-GAM+sgd}, the GAM term can be easily calculated via the Hessian vector product, which is an efficient approach to calculating the dot product between the Hessian and a vector without the need to calculate the entire Hessian. However, it can still introduce extra computation when calculated in each iteration. To accelerate the training with GAM, we investigate applying GAM to only a few iterations in each epoch. 
Surprisingly, we show that only several iterations of learning with GAM (with higher $\alpha$ compared with applying GAM to all iterations) improve model generalization considerably. As shown in Figure \ref{fig:time}, with approximately 1/20 of iterations, GAM considerablly improves test accuracy for both SGD and SAM on CIFAR-10 and CIFAR-100. When applying GAM to 1/10 iterations of training, it shows similar effectiveness to applying GAM to all the iterations, while the extra computational cost for GAM is less than 25\% of the original cost.
GAM outperformes SAM with lower computation overhead and achieves significant improvement when combined with SGD (the red line in the figure). When combined with SAM, GAM also improves generalization with low computation cost.
Thus the computation overhead of GAM can be easily controlled. The optimization of first-order flatness can be further accelerated by approximation of second-order gradient with first-order gradient and the details are in Appendix D.

\subsection{Visualization of Landscapes}
\label{app:vis}

\begin{figure}[t]
\vspace{-8pt}
    \centering
    \begin{subfigure}[b]{0.49\linewidth}
        \centering
        \includegraphics[width=0.8\linewidth]{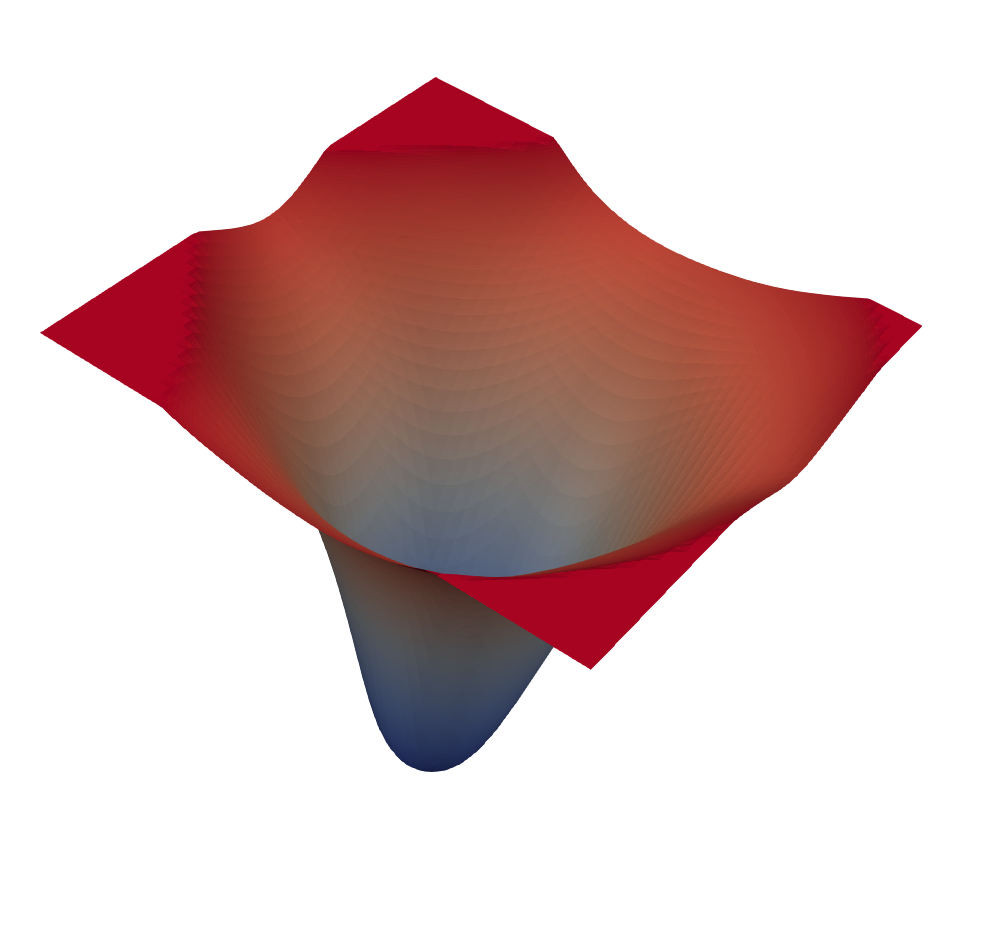}
        \vspace{-5px}
        \caption{SGD}
        \label{fig:sgd}
    \end{subfigure}
    \hfill
    \begin{subfigure}[b]{0.49\linewidth}
        \centering
        \includegraphics[width=0.8\linewidth]{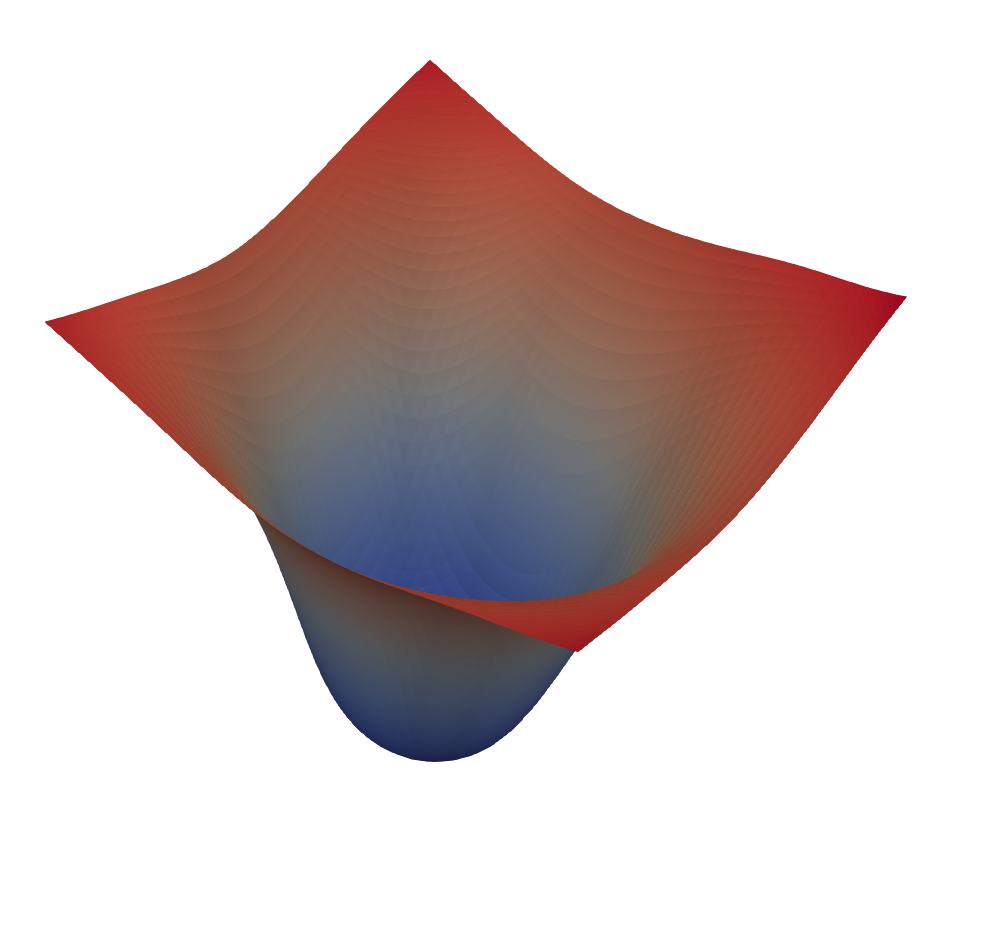}
        \vspace{-5px}
        \caption{SGD + GAM}
        \label{fig:sgd-GAM}
    \end{subfigure}
    \begin{subfigure}[b]{0.49\linewidth}
        \centering
        \includegraphics[width=0.8\linewidth]{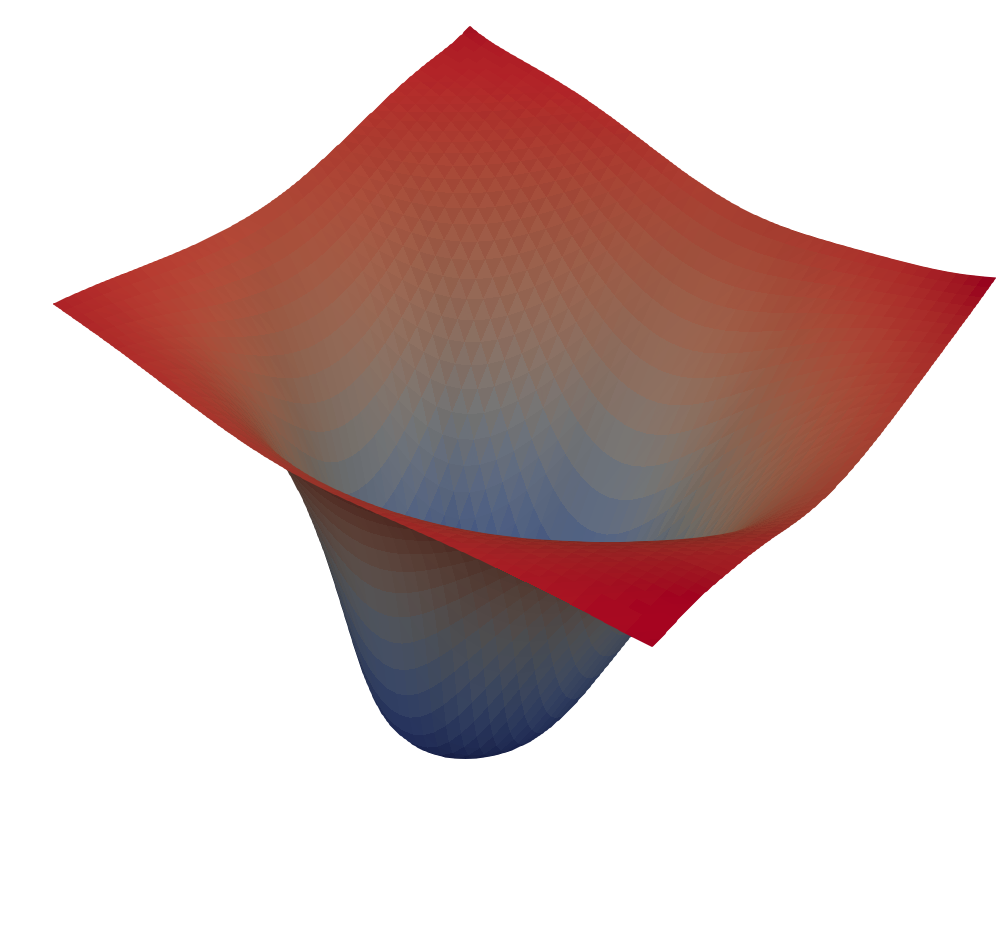}
        \vspace{-5px}
        \caption{SAM}
        \label{fig:sam}
    \end{subfigure}
    \hfill
    \begin{subfigure}[b]{0.49\linewidth}
        \centering
        \includegraphics[width=0.8\linewidth]{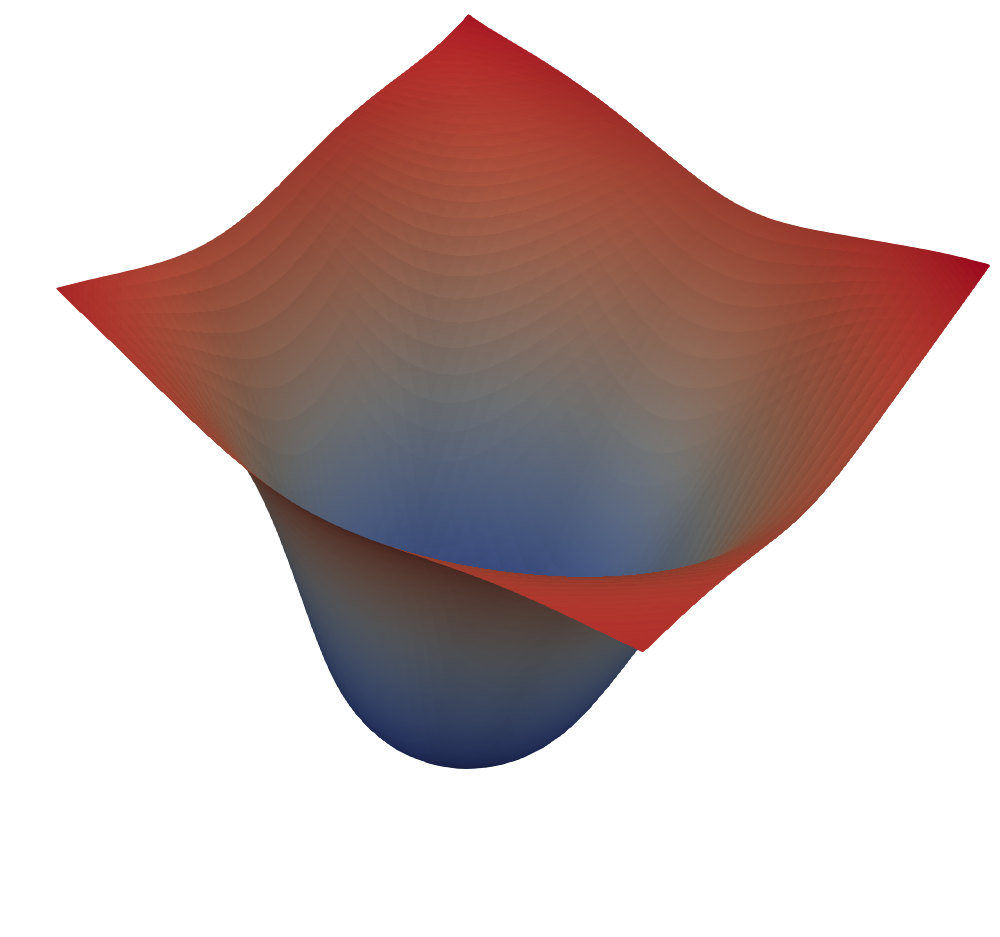}
        \vspace{-5px}
        \caption{SAM + GAM}
        \label{fig:sam-GAM}
    \end{subfigure}
    \caption{Visualization of loss landscape for SGD, SGD + GAM, SAM, SAM + GAM.}
    \label{fig:vis}
    \vspace{-15pt}
\end{figure}

We visualize the loss landscapes of models trained with SGD, SGD + GAM, SAM, SAM + GAM of the ResNet-18 model on CIFAR-100 following \cite{li2018visualizing}. All the models are trained with the same hyperparameters for 200 epochs as described in Section \ref{sect:cifar}. As shown in Figure \ref{fig:vis}, GAM consistently helps SGD and SAM find flatter minima.

%% file: paras/conclusion.tex
\section{Discussions}
We show that the most popular definitions of flatness, which we call the zeroth-order flatness, can be insufficient to indicate generalization error. Thus we propose first-order flatness, a stronger flatness measure that bounds both the maximum eigenvalue of Hessian and the zeroth-order flatness. We also propose a novel Gradient norm Aware Minimization (GAM) to optimize the first-order flatness. We empirically show that GAM considerably improves generalization for SGD, AdamW, and SAM.

Despite the empirical effectiveness of GAM, adopting the first-order flatness for generalization has the following limitations which could lead to potential future work. First, a theoretical explanation of whether a stronger flatness measure is better for generalization is vital for selecting flatness measures in practice. Second, the contribution to generalization of combining the zeroth-order and first-order flatness requires a thorough theoretical analysis.

%% file: paras/ack.tex
\section*{Acknowledgement}
This work was supported in part by National Key R\&D Program of China (No.
2018AAA0102004, No. 2020AAA0106300), National Natural Science Foundation of China (No.
U1936219, 62141607), Beijing Academy of Artificial Intelligence (BAAI). 

%% file: paras/appendix_method.tex
\section{Omitted details in Section \ref{sect:method}}
\subsection{Derivation of Equation \eqref{eq:GAM-derivative}} \label{sect:GAM-derivative}
We follow the steps in \cite{foret2021sharpness} to approximate
\begin{equation}
    \nabla \GAM(\bftheta) = \rho \cdot \nabla_{\bftheta} \max_{\bfepsilon \in B(0, \rho)}\left\|\nabla \hat{L}(\bftheta + \bfepsilon)\right\|.
\end{equation}
We first conduct the first-order Taylor expansion of $\|\nabla \hat{L}(\bftheta + \bfepsilon)\|$ and get that
\begin{equation}
    \begin{aligned}
        \bfepsilon^*(\bftheta) & = \argmax_{\bfepsilon \in B(0, \rho)} \left\|\nabla \hat{L}(\bftheta + \bfepsilon)\right\| \approx \argmax_{\bfepsilon \in B(0, \rho)} \left\|\nabla \hat{L}(\bftheta)\right\| + \left(\nabla \left\|\nabla \hat{L}(\bftheta)\right\|\right)^\top \bfepsilon \\
        & = \argmax_{\bfepsilon \in B(0, \rho)} \left(\nabla \left\|\nabla \hat{L}(\bftheta)\right\|\right)^\top \bfepsilon = \frac{\rho \cdot \bff}{\|\bff\|},
    \end{aligned}
\end{equation}
where $\bff = \nabla \|\nabla \hat{L}(\bftheta)\|$.
As a result, by letting $\bftheta^{\text{adv}} = \bftheta + \epsilon^*(\bftheta)$,
\begin{equation}
    \nabla \GAM(\bftheta) \approx \rho \cdot \nabla_{\bftheta} \left\|\nabla \hat{L}(\bftheta + \bfepsilon^*(\bftheta))\right\| = \rho \cdot \nabla \left\|\nabla \hat{L}(\bftheta^{\text{adv}})\right\| + \rho \cdot \frac{\mathrm{d} \bfepsilon^*(\bftheta)}{\mathrm{d} \bftheta} \cdot \nabla  \left\|\nabla \hat{L}(\bftheta^{\text{adv}})\right\|.
\end{equation}
In addition, similar to \cite{foret2021sharpness}, we further drop the second-order term to accelerate the computation. Finally, the derivative $\nabla \GAM(\bftheta)$ is given by
\begin{equation}
    \nabla \GAM(\bftheta) \approx \rho \cdot \nabla \left\|\nabla \hat{L}(\bftheta^{\text{adv}})\right\|, \quad \bftheta^{\text{adv}} = \bftheta + \rho \cdot \frac{\bff}{\|\bff\|}, \quad \bff = \nabla \left\|\nabla \hat{L}(\bftheta)\right\|.
\end{equation}

\section{Proofs}
\subsection{Proof of Lemma \ref{lemma:eigen-radius}}
\begin{proof}
    By assumption, we have that for all $\bftheta \in B(\bftheta^*, \rho)$,
    \begin{equation}
        \hat{L}(\bftheta) = \hat{L}(\bftheta^*) + \frac{1}{2}(\bftheta - \bftheta^*)^\top \left(\nabla^2 \hat{L}(\bftheta^*)\right)(\bftheta - \bftheta^*).
    \end{equation}
    In addition,
    \begin{equation}
        \nabla \hat{L}(\bftheta) = \left(\nabla^2 \hat{L}(\bftheta^*)\right)(\bftheta - \bftheta^*).
    \end{equation}
    As a result,
    \begin{equation}
        \max_{\bftheta \in B(\bftheta^*, \rho)} \left\|\nabla \hat{L}(\bftheta)\right\| = \max_{\bftheta \in B(\bftheta^*, \rho)} \left\|\left(\nabla^2 \hat{L}(\bftheta^*)\right)(\bftheta - \bftheta^*)\right\| = \rho \left\|\nabla^2 \hat{L}(\bftheta^*)\right\| = \rho \lambda_{\max}\left(\nabla^2 \hat{L}(\bftheta^*)\right).
    \end{equation}
    Now the claim follows.
\end{proof}

\subsection{Proof of Proposition \ref{prop:GAM-and-sam}}
\begin{proof}
    Suppose $\bfepsilon^* = \argmax_{\bfepsilon \in B(0,\rho)} \hat{L}(\bftheta + \bfepsilon)$. Then $\sam(\bftheta) = \hat{L}(\bftheta + \bfepsilon^*) - \hat{L}(\bftheta)$. According to the mean value theorem, there exists a constant $0 \le c \le 1$ such that
    \begin{equation}
        \hat{L}(\bftheta + \bfepsilon^*) - \hat{L}(\bftheta) = \left(\nabla \hat{L}(\bftheta + c\cdot \bfepsilon^*)\right)^\top\bfepsilon^*.
    \end{equation}
    As a result, by the Cauchy–Schwarz inequality,
    \begin{equation}
        \begin{aligned}
            \sam(\bftheta) & = \hat{L}(\bftheta + \bfepsilon^*) - \hat{L}(\bftheta) = \left(\nabla \hat{L}(\bftheta + c\cdot \bfepsilon^*)\right)^\top\bfepsilon^* \le \left\|\nabla \hat{L}(\bftheta + c\cdot \bfepsilon^*)\right\|\|\bfepsilon^*\| \\ & \le \max_{\bfepsilon \in B(0, \rho)} \left\|\nabla \hat{L}(\bftheta + \bfepsilon)\right\| \cdot \rho = \GAM(\bftheta).
        \end{aligned}
    \end{equation}
\end{proof}

\subsection{Proof of Proposition \ref{prop:main-bound}}
\begin{proof}
    Define $h(\bftheta) = \max_{\bftheta' \in B(\bftheta, \rho)}\left\|\nabla \hat{L}(\bftheta)\right\|$.
    Fix $\sigma = \rho / (\sqrt{d} + \sqrt{\log n})$, following the proof of Theorem 1 in \cite{foret2021sharpness}, we can obtain that with probability at least $1 - \delta$,
    \begin{small}
        \begin{equation} \label{eq:proof-main-1}
            \bbE_{\epsilon_i \sim N(0, \sigma^2)}\left[L(\bftheta+\bfepsilon)\right] \leq \bbE_{\epsilon_i \sim N(0, \sigma^2)}\left[\hat{L}(\bftheta+\bfepsilon)\right]+\sqrt{\frac{\frac{1}{4} d \log \left(1+\frac{\|\bftheta\|_2^2}{d \sigma^2}\right)+\frac{1}{4}+\log \frac{n}{\delta}+2 \log (6 n+3 d)}{n-1}}.
        \end{equation}
    \end{small}
    Since $\epsilon_i \sim N(0, \sigma^2)$, $\|\bfepsilon\|^2/\sigma^2$ has a chi-square distribution. As a result, according to \cite[Lemma 1]{laurent2000adaptive}, we have that for any $t > 0$,
    \begin{equation}
        \bbP\left(\|\bfepsilon\|^2/\sigma^2 - d \geq 2 \sqrt{ dt}+2 t\right) \leq \exp (-t).
    \end{equation}
    By letting $t = \frac{1}{2}\log n$, we can get that with probability at least $1 - 1 / \sqrt{n}$,
    \begin{equation}
      \|\bfepsilon\|^2 \le \sigma^2\left(d + \sqrt{2d \log n} + \log n\right) \le \sigma^2\left(\sqrt{d} + \sqrt{\log n}\right)^2 = \rho^2.
    \end{equation}
    As a result,
    \begin{equation} \label{eq:proof-main-2}
        \begin{aligned}
            & \, \bbE_{\epsilon_i \sim N(0, \sigma^2)}\left[\hat{L}(\bftheta+\bfepsilon)\right] \\
            \le & \, \bbE_{\epsilon_i \sim N(0, \sigma^2)}\left[\hat{L}(\bftheta+\bfepsilon) \mid \|\bfepsilon\| \le \rho\right] \bbP\left(\|\bfepsilon\| \le \rho\right) + \bbE_{\epsilon_i \sim N(0, \sigma^2)}\left[\hat{L}(\bftheta+\bfepsilon) \mid \|\bfepsilon\| > \rho\right] \bbP\left(\|\bfepsilon\| > \rho\right) \\
            \le & \, \bbE_{\epsilon_i \sim N(0, \sigma^2)}\left[\hat{L}(\bftheta+\bfepsilon) \mid \|\bfepsilon\| \le \rho\right] + \frac{M}{\sqrt{n}}. \\
        \end{aligned}
    \end{equation}
    According to the mean value theorem and Cauchy–Schwarz inequality, for any $\bfepsilon$ such that $\|\bfepsilon\| < \rho$, there exists a constant $0 \le c \le 1$, such that
    \begin{equation} \label{eq:proof-main-3}
        \hat{L}(\bftheta + \bfepsilon) = \hat{L}(\bftheta) + \left(\nabla \hat{L}(\bftheta + c \bfepsilon)\right)^\top\bfepsilon \le \hat{L}(\bftheta) + \left\|\nabla \hat{L}(\bftheta + c \bfepsilon)\right\|\cdot\|\bfepsilon\| \le \hat{L}(\bftheta) + h(\bftheta)\rho = \hat{L}(\bftheta) + \GAM(\bftheta).
    \end{equation}
    Now the claim follows from Equations \eqref{eq:proof-main-1}, \eqref{eq:proof-main-2}, and \eqref{eq:proof-main-3}.
\end{proof}

\subsection{Proof of Theorem \ref{thrm:convergence-main}}
\begin{proof}
    Observe that
    \begin{equation}
        \begin{aligned}
            & \, \left\|\nabla \Loverall(\bftheta_t)\right\|^2 = \left\|\nabla \Loracle(\bftheta_t) + \alpha \rho_t \cdot \nabla \left\|\nabla \hat{L} (\bftheta_t^{\text{adv}})\right\|\right\|^2 \\
            \le & \, 2\left(\left\|\nabla \Loracle(\bftheta_t)\right\|^2 + \left\|\alpha \rho_t \cdot \nabla \left\|\nabla \hat{L} (\bftheta_t^{\text{adv}})\right\|\right\|^2\right).
        \end{aligned}
    \end{equation}
    The claim follows from Propositions \ref{prop:convergence-L} and \ref{prop:convergence-perturbation}.
\end{proof}

\begin{proposition} \label{prop:convergence-L}
    Assume the conditions in Theorem \ref{thrm:convergence-main} hold (with parameters $\gamma_1, \gamma_2, G^{\text{loss}}, G^{\text{norm}}, \tilde{G}^{\text{loss}}, M, \eta_0, \rho_0, \alpha$). Then with learning rate $\eta_t = \eta_0 / \sqrt{t}$ and perturbation radius $\rho_t = \rho_0 / \sqrt{t}$, Algorithm \ref{alg:GAM} could obtain
    \begin{equation}
        \frac{1}{T}\sum_{t=1}^{T} \bbE\left[\left\|\nabla \Loracle(\bftheta_t)\right\|^2\right] \le \frac{C_1' + C_2' \log T}{\sqrt{T}}
    \end{equation}
    for some constants $C_1'$ and $C_2'$ that only depend on $\gamma_1, \gamma_2, G^{\text{loss}}, G^{\text{norm}}, \tilde{G}^{\text{loss}}, M, \eta_0, \rho_0, \alpha$.
\end{proposition}
\begin{proof}
    By definition, we have $\bfh_t^{\text{loss}} =  \tilde{g}_t^{\text{loss}}(\bftheta_t)$ and $\bfh_t^{\text{norm}} = g_t^{\text{norm}}(\bftheta_t^{\text{adv}})$. By assumption,
    \begin{equation}
        \begin{aligned}
            \Loracle(\bftheta_{t+1}) & \le \Loracle(\bftheta_{t}) + \left(\nabla \Loracle(\bftheta_{t})\right)^\top(\bftheta_{t+1} - \bftheta_{t}) + \frac{\gamma_1}{2}\left\|\bftheta_{t+1} - \bftheta_{t}\right\|^2 \\
            & = \Loracle(\bftheta_{t}) - \eta_t \left(\nabla \Loracle(\bftheta_{t})\right)^\top\left(\bfh^{\text{loss}}_t + \alpha \rho_t \bfh^{\text{norm}}_t\right) + \frac{\gamma_1\eta_t^2}{2}\left\|\bfh^{\text{loss}}_t + \alpha \rho_t \bfh^{\text{norm}}_t\right\|^2.
        \end{aligned}
    \end{equation}

    Take the expectation conditioned on the observations till timestamp $t$. By the assumption $\bbE[\bfh_t^{\text{loss}}] = \bbE[\tilde{g}_t^{\text{loss}}(\bftheta_{t})] = \nabla \Loracle(\bftheta_{t})$ and $\bbE[\bfh_t^{\text{norm}}] = \bbE[g_t^{\text{norm}}(\bftheta_{t}^{\text{adv}})]$, we can obtain that
    \begin{equation} \label{eq:proof-convergence-loss-1}
        \begin{aligned}
            & \bbE\left[\Loracle(\bftheta_{t+1})\right] - \Loracle(\bftheta_t) \\
            \le & \, - \eta_t \left\|\nabla \Loracle(\bftheta_t)\right\|^2 - \eta_t\rho_t\alpha\left(\nabla \Loracle(\bftheta_t)\right)^\top\bbE\left[g_t^{\text{norm}}(\bftheta_{t}^{\text{adv}})\right] + \frac{\gamma_1\eta_t^2}{2}\left\|\bfh^{\text{loss}}_t + \alpha \rho_t \bfh^{\text{norm}}_t\right\|^2 \\
        \end{aligned}
    \end{equation}
    We have
    \begin{equation} \label{eq:proof-convergence-loss-2}
        - \eta_t\rho_t\alpha\left(\nabla \Loracle(\bftheta_t)\right)^\top\bbE\left[g_t^{\text{norm}}(\bftheta_{t}^{\text{adv}})\right] \le \eta_t\rho_t\alpha\left\|\nabla \Loracle(\bftheta_t)\right\|\left\|\bbE\left[g_t^{\text{norm}}(\bftheta_{t}^{\text{adv}})\right]\right\| \le \eta_t\rho_t\alpha \tilde{G}^{\text{loss}}G^{\text{norm}}.
    \end{equation}
    In addition,
    \begin{equation} \label{eq:proof-convergence-loss-3}
        \bbE\left[\left\|\bfh^{\text{loss}}_t + \alpha \bfh^{\text{norm}}_t\right\|^2\right] \le 2\bbE\left[\left\|\bfh_t^{\text{loss}}\right\|^2\right] + 2\alpha^2 \bbE\left[\left\|\bfh_t^{\text{norm}}\right\|^2\right] \le 2\left(\tilde{G}^{\text{loss}}\right)^2 + 2\alpha^2 \left(G^{\text{norm}}\right)^2.
    \end{equation}
    Combining Equations \eqref{eq:proof-convergence-loss-1}, \eqref{eq:proof-convergence-loss-2}, and \eqref{eq:proof-convergence-loss-3}, we can get that
    \begin{equation}
        \eta_t \left\|\nabla \Loracle(\bftheta_{t})\right\|^2 \le - \bbE\left[\Loracle(\bftheta_{t+1})\right] + \Loracle(\bftheta_{t})  + \eta_t \rho_t Z_1 + \eta_t^2 Z_2
    \end{equation}
    for some constants $Z_1$ and $Z_2$ that only depend on $\gamma_1, \gamma_2, G^{\text{loss}}, G^{\text{norm}}, \tilde{G}^{\text{loss}}, \alpha$. Now perform telescope sum and take the expectations at each step, we can obtain that
    \begin{equation}
        \sum_{t=1}^{T}\eta_{t} \left\|\nabla \Loracle(\bftheta_{t})\right\|^2 \le -\bbE\left[\Loracle(\bftheta_{T+1})\right] + \Loracle(\bftheta_1) + Z_1 \sum_{t=1}^T\eta_t\rho_t + Z_2 \sum_{t=1}^T \eta_t^2.
    \end{equation}
    By letting $\eta_t = \eta_0 / \sqrt{t}$ and $\rho_t = \rho_0 / \sqrt{t}$, we can get that
    \begin{equation}
        \begin{aligned}
            \frac{\eta_0}{\sqrt{T}}\sum_{t=1}^{T}\left\|\nabla \Loracle(\bftheta_{t})\right\|^2 & \le \sum_{t=1}^{T}\eta_{t} \left\|\nabla \Loracle(\bftheta_{t})\right\|^2 \\
            & \le -\bbE\left[\Loracle(\bftheta_{T+1})\right] + \Loracle(\bftheta_1) + Z_1 \sum_{t=1}^T\eta_t\rho_t + Z_2 \sum_{t=1}^T \eta_t^2 \\
            & \le 2M + Z_1\eta_0\rho_0 \sum_{t=1}^T \frac{1}{t} + Z_2\eta_0^2 \sum_{t=1}^T\frac{1}{t} \\
            & \le Z_4 + Z_5 \log T
        \end{aligned}
    \end{equation}
    for some constants $Z_4$ and $Z_5$ that only depend on $\gamma_1, \gamma_2, G^{\text{loss}}, G^{\text{norm}}, \tilde{G}^{\text{loss}}, M, \eta_0, \rho_0, \alpha$. Divide the two sides of the equation by $\eta_0\sqrt{T}$ and the claim follows.
\end{proof}

\begin{proposition} \label{prop:convergence-perturbation}
    Assume the conditions in Theorem \ref{thrm:convergence-main} hold (with parameters $\gamma_1, \gamma_2, G^{\text{loss}}, G^{\text{norm}}, \tilde{G}^{\text{loss}}, M, \eta_0, \rho_0, \alpha$). Then with perturbation radius $\rho_t = \rho_0 / \sqrt{t}$, Algorithm 1 could obtain
    \begin{equation}
        \frac{1}{T}\sum_{t=1}^{T} \bbE\left[\left\|\alpha \rho_t \cdot \nabla \left\|\nabla \hat{L} (\bftheta^{\text{adv}})\right\|\right\|^2\right] \le \frac{C_1'' + C_2'' \log T}{\sqrt{T}}
    \end{equation}
    for some constants $C_1''$ and $C_2''$ that only depend on $\gamma_2, \rho_0, \alpha$.
\end{proposition}
\begin{proof}
    For any $t \in \{1, 2, \dots, T\}$,
    \begin{equation}
        \begin{aligned}
            & \, \bbE\left[\left\|\alpha \rho_t \cdot \nabla \left\|\nabla \hat{L} (\bftheta^{\text{adv}})\right\|\right\|^2\right] = \alpha^2 \rho_t ^2  \bbE\left[\left\|\nabla \left\|\nabla \hat{L} (\bftheta^{\text{adv}})\right\|\right\|^2\right] \\
            = & \, \alpha^2 \rho_t^2 \bbE\left[\left\|\nabla^2 \hat{L}(\bftheta_t^{\text{adv}}) \cdot \frac{\nabla \hat{L}(\bftheta_t^{\text{adv}})}{\left\|\nabla \hat{L}(\bftheta_t^{\text{adv}})\right\|}\right\|\right] \le \alpha^2 \rho_t^2 \bbE\left[\left\|\nabla^2 \hat{L}(\bftheta_t^{\text{adv}})\right\|\left\| \frac{\nabla \hat{L}(\bftheta_t^{\text{adv}})}{\left\|\nabla \hat{L}(\bftheta_t^{\text{adv}})\right\|}\right\|\right] \\
            \le & \, \alpha^2 \rho_t^2 \bbE[\gamma_2] = \alpha^2 \rho_t^2 \gamma_2.
        \end{aligned}
    \end{equation}
    By letting $\rho_t = \rho_0 / \sqrt{t}$,
    \begin{equation}
        \frac{1}{T}\sum_{t=1}^{T} \bbE\left[\left\|\alpha \rho_t \cdot \nabla \left\|\nabla \hat{L} (\bftheta^{\text{adv}})\right\|\right\|^2\right] \le \frac{1}{T}\alpha^2\gamma_2\rho_0^2\sum_{t=1}^T\frac{1}{t} \le \frac{C_1'' + C_2'' \log T}{\sqrt{T}} 
    \end{equation}
    for some constants $C_1''$ and $C_2''$ that only depend on $\gamma_2, \rho_0, \alpha$.
\end{proof}

%% file: paras/appendix_exp.tex
\section{More Experimental Results and Details}

\subsection{More Results on Training from Scratch}
Due to space limitation, we omit the experimental results on CIFAR-10 and CIFAR-100 \cite{krizhevsky2009learning} with ResNeXt \cite{xie2017aggregated}, DenseNet \cite{huang2019convolutional} and ViTs \cite{dosovitskiy2020image} in Section 5.2.1 in the main paper and report them in Section \ref{app:exp:cifar}. 
Then we report the results of robustness to label noise in Section \ref{app:exp:noise}.

\subsubsection{CIFAR-10 and CIFAR-100}
\label{app:exp:cifar}

We report the omitted results on CIFAR-10 and CIFAR-100 with ResNeXt, DenseNet and ViTs. 
As described in the main paper, all the models are trained for 200 epochs from scratch. We evaluate GAM both with basic data augmentations (i.e., horizontal flip, padding by four pixels, and random crop) and advanced data augmentation including cutout regularization \cite{devries2017improved}, RandAugment \cite{cubuk2018autoaugment} and AutoAugment \cite{cubuk2020randaugment}. The hyperparameters, $\rho$ and $\alpha$ are searched with the same approach described in the main paper.

\begin{table*}[t]
    \centering
    \small
    \caption{Results of GAM with state-of-the-art models on CIFAR-10 and CIFAR-100. The best results are highlighted in bold font.}
    \resizebox{\linewidth}{!}{
    \begin{tabular}{c|c|cc|cc|cc|cc}
        \toprule
        & & \multicolumn{4}{c|}{CIFAR-10} & \multicolumn{4}{c}{CIFAR-100} \\
        \midrule
        Model & Aug & SGD & SGD + GAM & SAM & SAM + GAM & SGD & SGD + GAM & SAM & SAM + GAM \\
        \midrule
        DenseNet121 & Basic & 91.16$_{\pm 0.13}$ & \textbf{92.35}$_{\pm 0.14}$ & 92.19$_{\pm 0.20}$ & \textbf{92.72}$_{\pm 0.30}$ & 69.25$_{\pm 0.40}$ & \textbf{70.48}$_{\pm 0.27}$ & 70.44$_{\pm 0.19}$ & \textbf{71.16}$_{\pm 0.25}$ \\
        DenseNet121 & Cutout & 91.85$_{\pm0.17}$ & \textbf{92.93}$_{\pm 0.26}$ & 92.35$_{\pm 0.16}$ & \textbf{93.30}$_{\pm 0.17}$ & 70.17$_{\pm 0.31}$ & \textbf{71.47}$_{\pm 0.23}$ & 70.89$_{\pm 0.15}$ & \textbf{71.80}$_{\pm 0.07}$ \\
        DenseNet121 & RA & 91.59$_{\pm0.16}$ & \textbf{92.37}$_{\pm 0.20}$ & 92.32$_{\pm 0.29}$ & \textbf{92.97}$_{\pm 0.27}$ & 69.65$_{\pm 0.36}$ & \textbf{70.10}$_{\pm 0.26}$ & 70.49$_{\pm 0.16}$ & \textbf{71.43}$_{\pm 0.17}$ \\
        DenseNet121 & AA & 92.65$_{\pm0.10}$ & \textbf{94.17}$_{\pm 0.25}$ & 92.96$_{\pm 0.19}$ & \textbf{94.05}$_{\pm 0.22}$ & 70.53$_{\pm 0.17}$ & \textbf{72.25}$_{\pm 0.19}$ & 71.34$_{\pm 0.20}$ & \textbf{72.90}$_{\pm 0.17}$ \\
        \midrule
        ResNeXt29-32x4d & Basic & 95.75$_{\pm0.31}$ & \textbf{96.46}$_{\pm0.25}$ & 96.32$_{\pm0.36 }$ & \textbf{96.90}$_{\pm 0.24}$ & 79.45$_{\pm 0.29}$ & \textbf{81.67}$_{\pm 0.26}$ & 81.35$_{\pm 0.12}$ & \textbf{82.93}$_{\pm 0.25}$ \\
        ResNeXt29-32x4d & Cutout & 96.20$_{\pm0.37}$ & \textbf{97.82}$_{\pm0.24}$ & 96.44$_{\pm0.21 }$ & \textbf{97.85}$_{\pm0.27}$ & 80.56$_{\pm0.20 }$ & \textbf{82.62}$_{\pm0.33 }$ & 82.49$_{\pm0.25 }$ & \textbf{83.58}$_{\pm0.09}$ \\
        ResNeXt29-32x4d & RA & 95.86$_{\pm0.28}$ & \textbf{97.17}$_{\pm 0.26}$ & 96.75$_{\pm 0.35}$ & \textbf{97.79}$_{\pm0.19 }$ & 79.88$_{\pm0.12 }$ & \textbf{81.75}$_{\pm0.23 }$ & 82.26$_{\pm0.15 }$ & \textbf{83.02}$_{\pm0.22 }$ \\
        ResNeXt29-32x4d & AA & 96.58$_{\pm0.18}$ & \textbf{97.46}$_{\pm 0.15}$ & 97.38$_{\pm0.25 }$ & \textbf{97.58}$_{\pm0.16 }$ & 80.47$_{\pm0.13 }$ & \textbf{82.02}$_{\pm0.19 }$ & 81.52$_{\pm0.26 }$ & \textbf{83.35}$_{\pm0.09 }$ \\
        \midrule
        ViT-S/16 & Basic & 95.27$_{\pm0.23}$ & \textbf{97.21}$_{\pm 0.14}$ & 96.85$_{\pm0.25}$ & \textbf{97.58}$_{\pm 0.20}$ &  79.52$_{\pm 0.36}$ & \textbf{83.35}$_{\pm 0.28}$ & 82.77$_{\pm0.29 }$ & \textbf{84.30}$_{\pm0.25}$ \\
        ViT-S/16 & Cutout & 95.36$_{\pm0.22}$ & \textbf{97.53}$_{\pm 0.17}$ & 97.10$_{\pm0.25 }$ & \textbf{97.85}$_{\pm0.10 }$ & 79.36$_{\pm0.21 }$ & \textbf{83.59}$_{\pm0.28 }$ & 82.86$_{\pm0.14 }$ & \textbf{84.53}$_{\pm0.16 }$ \\
        ViT-S/16 & RA & 95.59$_{\pm0.19}$ & \textbf{97.44}$_{\pm 0.31}$ & 97.18$_{\pm0.12 }$ & \textbf{97.59}$_{\pm0.11 }$ & 79.96$_{\pm0.22 }$ & \textbf{83.80}$_{\pm0.20 }$ & 83.36$_{\pm0.13 }$ & \textbf{84.66}$_{\pm0.23 }$ \\
        ViT-S/16 & AA & 96.40$_{\pm0.30}$ & \textbf{97.82}$_{\pm 0.21}$ & 97.52$_{\pm0.25 }$ & \textbf{97.97}$_{\pm0.13 }$ & 80.35$_{\pm0.06 }$ & \textbf{84.02}$_{\pm0.18 }$ & 83.54$_{\pm0.19 }$ & \textbf{85.20}$_{\pm0.26}$ \\ 
        \bottomrule
    \end{tabular}}
    \label{app:tab:cifar}
\end{table*}

Results are shown in Table \ref{app:tab:cifar}. GAM consistently improves generalization for all models. 
We observe the same results as in the main paper.
When combined with SGD, GAM achieves considerably higher test accuracy compared with SGD. And GAM also achieves improvement when combined with SAM. 

\paragraph{Comparison with GNP}
GNP can be considered as a special case of GAM where $\rho$ is set to 0.
We compare GAM with GNP \cite{zhao2022penalizing} in on CIFAR-100 in Table \ref{tab:GNP}. We follow hyperparameters searching and choice of GNP in its original paper. GAM consistently outperforms GNP by noticeable margins. 

\begin{table}[ht]
    \centering
    \caption{Comparison with GNP on CIFAR-100. \textit{-BA} indicates the basic data augmentation, \textit{-CU} indicates cutout regularization, \textit{-RA} indicates RandAugment, and \textit{-AA} indicates AutoAugment.}
    \label{tab:GNP}
    \resizebox{\linewidth}{!}{
    \begin{tabular}{cccccccccc}
        \toprule
        & Res18-BA & Res18-CU & Res18-RA & Res18-AA & Res101-BA & Res101-CU & Res101-RA & Res101-AA  \\
        \midrule
        SGD  & $78.32_{\pm 0.32}$ &
        $78.73_{\pm 0.13}$ & 78.62$_{\pm 0.32}$ & $78.88_{\pm 0.15}$ & $80.47_{\pm 0.13}$ & $80.53_{\pm 0.30}$ & $80.60_{\pm 0.28}$ & $81.83_{\pm 0.37}$\\
        
        GNP + SGD & $78.80_{\pm 0.40}$ & $79.29_{\pm 0.15}$ & $79.21_{\pm 0.27}$ & $80.29_{\pm 0.05}$ & $81.17_{\pm 0.29}$ & $81.10_{\pm 0.14}$ & $81.31_{\pm 0.88}$ & $82.53_{\pm 0.25}$ \\
        
        GAM + SGD  & \textbf{79.53$_{\pm 0.30}$ }& \textbf{79.89$_{\pm 0.31}$} & \textbf{79.82$_{\pm 0.24}$} & \textbf{80.56$_{\pm 0.21}$} & \textbf{82.21$_{\pm 0.40}$} & \textbf{82.36$_{\pm 0.24}$} & \textbf{82.40$_{\pm 0.31}$} & \textbf{83.19$_{\pm 0.15}$} \\
        \midrule
        & WRN10-BA & WRN10-CU & WRN10-RA & WRN10-AA & Pyr110-BA & Pyr110-CU & Pyr110-RA & Pyr110-AA  \\
        \midrule
        SGD  & $81.40_{\pm 0.13}$ & $81.53_{\pm 0.40}$ & $81.65_{\pm 0.18}$ & $81.99_{\pm 0.11}$ & $82.74_{\pm 0.12}$ & $83.31_{\pm 0.21}$ & $84.04_{\pm 0.19}$ & $84.48_{\pm 0.03}$\\
        
        GNP + SGD & $82.30_{\pm 0.05}$ & $82.54_{\pm 0.19}$ & $82.99_{\pm 0.39}$ & $83.58_{\pm 0.32}$ & $83.99_{\pm 0.27}$ & $84.46_{\pm 0.16}$ & $84.47_{\pm 0.08}$ & $84.83_{\pm 0.21}$ \\
        
        GAM + SGD  & \textbf{83.45$_{\pm 0.09}$} & \textbf{83.69$_{\pm 0.08}$} & \textbf{83.84$_{\pm 0.09}$} & \textbf{84.02$_{\pm 0.18}$} & \textbf{84.91$_{\pm 0.09}$} & \textbf{85.20$_{\pm 0.19}$} & \textbf{86.47$_{\pm 0.14}$} & \textbf{85.92$_{\pm 0.03}$} \\
        \bottomrule
    \end{tabular}}
    \label{tab:domainnet}
\end{table}

\subsection{Robustness to Label Noise}
\label{app:exp:noise}
It is observed that sharpness-aware minimization methods are robust to perturbations to label noise \cite{foret2021sharpness,kwon2021asam}. Here we assess the degree of robustness that GAM provides to label noise. 

Following \cite{foret2021sharpness,kwon2021asam}, we measure the effectiveness of GAM in the classical noisy-label setting for CIFAR-10. A fraction of the training data labels is randomly flipped \cite{jiang2020beyond} while the test data remains unmodified. We train a ResNet32 for 200 epochs following \cite{jiang2020beyond}. Hyperparameter settings for all the models are the same as that of previous CIFAR experiments. Following \cite{arazo2019unsupervised,kwon2021asam}, we report the best results during the training instead of the results at the end of the training.

We report test accuracies for SGD, SAM, SGD+GAM, and SAM+GAM obtained from 3 independent runs for each label noise level in Table \ref{app:tab:noise}.
As seen in Table \ref{app:tab:noise}, GAM shows a high degree of robustness to label noise. GAM consistently improves the robustness to label noise for both SGD and SAM. 

\begin{table*}[t]
    \centering
    \small
    \caption{Test accuracy of ResNet32 on CIFAR-10 with label noise.}
    \begin{tabular}{c|cc|cc}
        \toprule
         Noise Rate (\%) & Base Opt & Base + GAM & SAM & SAM + GAM \\
        \midrule
        0\% & 95.71 & \textbf{96.55} & 96.25 & \textbf{96.88} \\
        20\% & 92.05 & \textbf{94.20} & 93.85 & \textbf{94.74} \\
        40\% & 88.89 & \textbf{92.17} & 90.85 & \textbf{92.57} \\
        60\% & 83.17 & \textbf{88.49} & 87.37 & \textbf{89.65} \\
        80\% & 63.16 & \textbf{73.82} & 70.65 & \textbf{76.33} \\
        \bottomrule
    \end{tabular}
    \label{app:tab:noise}
\end{table*}

\subsection{Detailed Results and Discussions about Computation Overhead}
\begin{table*}[t]
    \centering
    \caption{Accuracy and training speed of training with different ratios of iterations using GAM. Superscripts indicate the ratio of iterations in each epoch is trained with GAM (e.g., GAM$^{0.05}$ indicates that 5\% of iterations are trained with GAM, while the remaining iterations are trained with the basic optimizer). Numbers in parentheses indicate the ratio of the training speed compared with the vanilla base optimizer SGD/SAM. We mark runs whose training speed is lower than 50\% of the basic optimizer in red and others in green. Please note that the speed of SAM is about 50\% w.r.t SGD's speed. Thus when combined with SGD, green markers indicate that the speed of GAM under the corresponding ratio is faster than SAM.} 
    \resizebox{0.9\linewidth}{!}{
    \begin{tabular}{c|c|ccccc}
        \toprule
        \multirow{7}{*}{CIFAR-10} & & SGD &  SGD + GAM$^{0.05}$ & SGD + GAM$^{0.1}$ & SGD + GAM$^{0.5}$ & SGD + GAM$^{1}$ \\
        \cmidrule{2-7}
        & Accuracy & 95.32 & 96.08 & 96.15 & 96.17 & 96.17 \\
        & Images/s & 2,593 (\color{teal}{100\%}) & 2,258 (\color{teal}{87\%}) & 1,996 (\color{teal}{77\%})  & 1,023 (\color{BrickRed}{39\%}) & 658 (\color{BrickRed}{25\%}) \\
        \cmidrule{2-7}
        & & SAM & SAM + GAM$^{0.05}$ & SAM + GAM$^{0.1}$ & SAM + GAM$^{0.5}$ & SAM + GAM$^{1}$ \\
        \cmidrule{2-7}
        & Accuracy & 96.10 & 96.54 & 96.62 & 96.65 & 96.58 \\
        & Images/s & 1,314 (\color{teal}{100\%}) & 1,247 (\color{teal}{95\%}) & 1,184 (\color{teal}{90\%}) & 858 (\color{teal}{65\%}) & 629 (\color{BrickRed}{48\%}) \\
        \midrule
        \multirow{7}{*}{CIFAR-100} & & SGD & SGD + GAM$^{0.05}$ & SGD + GAM$^{0.1}$ & SGD + GAM$^{0.5}$ & SGD + GAM$^{1}$ \\
        \cmidrule{2-7}
        & Accuracy & 78.32 & 79.25 & 79.42 & 79.50 & 79.53 \\
        & Images/s & 2,609 (\color{teal}{100\%}) & 2,243 (\color{teal}{86\%}) & 1,955 (\color{teal}{75\%}) & 1,011 (\color{BrickRed}{39\%}) & 655 (\color{BrickRed}{25\%}) \\
        \cmidrule{2-7}
        & & SAM & SAM + GAM$^{0.05}$ & SAM + GAM$^{0.1}$ & SAM + GAM$^{0.5}$ & SAM + GAM$^{1}$ \\
        \cmidrule{2-7}
        & Accuracy & 79.27 & 80.08 & 80.44 & 80.40 & 80.45 \\
        & Images/s & 1,318 (\color{teal}{100\%}) & 1,251 (\color{teal}{95\%}) & 1,172 (\color{teal}{89\%}) & 848 (\color{teal}{64\%}) & 628 (\color{BrickRed}{48\%})\\
        \bottomrule
    \end{tabular}}
    \label{app:tab:time}
\end{table*}

Here we report the detailed results of the trade-off between computation overhead and test accuracy of GAM. As discussed in Section 4.3 and 5.5 in the main paper, the GAM term can be easily calculated via the Hessian vector product, which is an efficient approach to calculating the dot product between the Hessian and a vector without the need to calculate the entire Hessian. And we also notice that only several iterations of learning with GAM (with higher $\alpha$ compared with applying GAM to all iterations) improve model generalization considerably. As seen in Table \ref{app:tab:time}, applying GAM to 1/10 iterations of training shows similar generalization performance to applying GAM to all the iterations, while the extra computational cost for GAM is less than 25\% of the original cost. Thus the computation overhead of GAM can be easily controlled.

\subsection{Ablation Study}
GAM has two hyperparameters, $\rho$, and $\alpha$. We analyze the influence of the choice of them in the following subsection \ref{sec:ablation1} and \ref{sec:ablation2}. The results are shown in Figure \ref{fig:ablation-hyperparameter}.

\begin{figure}
    \centering
    \begin{subfigure}[b]{0.49\textwidth}
        \centering
        \includegraphics[width=0.7\linewidth]{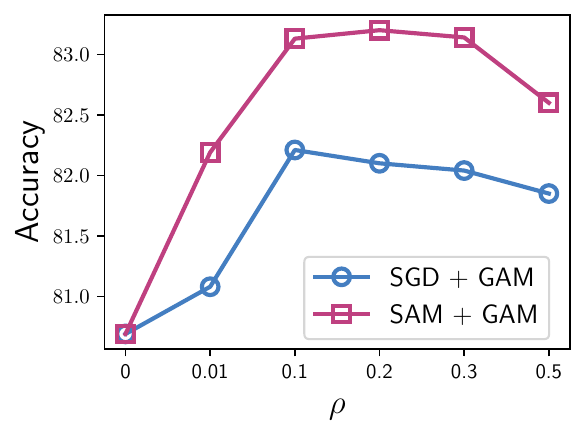}
        \caption{}
        \label{fig:hyper-rho}
    \end{subfigure}
    \hfill
    \begin{subfigure}[b]{0.49\textwidth}
        \centering
        \includegraphics[width=0.7\linewidth]{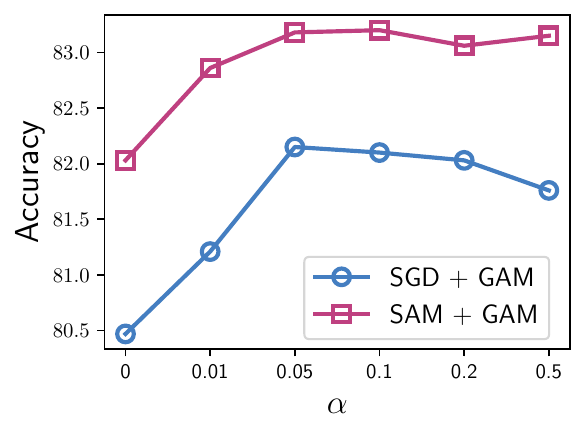}
        \caption{}
        \label{fig:hyper-alpha}
    \end{subfigure}
    \hfill
    \caption{The influence of hyperparameters $\rho$ and $\alpha$ on the performance of ResNet101 on CIFAR-100.}
    \label{fig:ablation-hyperparameter}
\end{figure}

\subsubsection{Influence of $\rho$}
\label{sec:ablation1}
$\rho$ controls the step length of gradient ascent in GAM. When $\rho$ is set to 0, GAM degenerates into a naive regularizer constraining the gradient norm at each training step.  
We plot the performance of ResNet101 on CIFAR-100 with varying $\rho$. For experiments where GAM is combined with SAM, we apply the same $\rho$ for both GAM and SAM. As shown in Figure \ref{fig:hyper-rho}, GAM with $\rho$ larger than 0 outperforms GAM without gradient ascent, showing that the gradient ascent is necessary for GAM. 
Moreover, GAM consistently outperforms SGD with different $\rho$. 
GAM also consistently improves SAM's performance with various $\rho$, indicating that the first-order flatness can improve the generalization ability of zero-order flatness.

\subsubsection{Influence of $\alpha$}
\label{sec:ablation2}
$\alpha$ controls the strength of GAM penalty. When $\alpha$ is set to 0 GAM degenerates into the basic optimizer (SGD or SAM).
We show the performance of GAM with varying $\alpha$ in Figure \ref{fig:hyper-alpha}. Compared with SGD, GAM shows considerable improvement with varying $\alpha$. The improvement of GAM under various $\alpha$ is also observed when combined with SAM.

\subsection{Training Details and Selection of Hyperparameters}
\subsubsection{Training Details of Training from Scratch Experiments}
\begin{table}
\centering
\caption{Hyperparameters for Algorithm \ref{alg:GAM} on CIFAR-10 and CIFAR-100 datasets.}
\resizebox{0.7\linewidth}{!}{
\begin{tabular}{l|ccccc}
\toprule
Model  & Learning Rate & Weight Decay & Base Optimizer & Epochs & LR Schedule \\
\midrule
ResNet18  & 0.1 & 0.0005 & SGD & 200 & Cosine \\
ResNet101   & 0.1 & 0.0005 & SGD & 200 & Cosine \\
WRN28\_2   & 0.1 & 0.0005 & SGD &  200 & Cosine \\
WRN28\_10   & 0.1 & 0.0005 & SGD &  200 & Cosine  \\
PyramidNet110   & 0.05 & 0.0005 & SGD  &  200 & Cosine  \\
DenseNet121   & 0.1 & 0.001 & SGD  & 200 & Cosine \\
ResNeXt29-32x4d  & 0.1 & 0.0005 & SGD &  200 & Cosine \\

\bottomrule

\end{tabular}
}
\label{app:tab:hp1}
\end{table}

\begin{table}
\centering
\caption{Hyperparameters for Algorithm \ref{alg:GAM} on ImageNet.}
\begin{tabular}{l|ccccccc}
\toprule
Model & $\rho$ & $\alpha$ & Learning Rate & Weight Decay & Base Optimizer & Epochs & LR Schedule \\
\midrule
ResNet50 & 0.2 & 0.1 & 0.1 & 0.0001 & SGD & 90 & Cosine \\
ResNet101 &  0.2 & 0.1  & 0.1 & 0.0001 & SGD & 90 & Cosine \\
\midrule
ViT-S/32 & 0.3 & 0.5 & 0.0003 & 0.3 & AdamW & 300 & Cosine \\
ViT-B/32 & 0.3 & 0.5 & 0.0003 & 0.3 & AdamW  & 300 & Cosine \\
\bottomrule
\end{tabular}
\label{app:tab:hp2}
\end{table}

We search hyperparameters, including learning rate and weight decay for all the models unless otherwise noted. 
For ResNets, we conduct a grid search of learning rate in \{0.01, 0.1, 1.0\} and weight decay in \{0.0001, 0.0005, 0.001, 0.01, 0.1\}. The batch size is set to 128 for all models.
For Vits, we search the learning rate in \{1e-3, 3e-3, 1e-2, 3e-3\}, and weight decay in \{0.001, 0.01, 0.1\}. We adopt SGD with momentum = 0.9 for ResNets and AdamW with $\beta_1$ = 0.9, $\beta_2$ = 0.999 for ViTs. We train ResNets for 90 epochs, and train ViTs for 300 epochs following \cite{chen2021vision,zhuang2022surrogate}. We first search for the optimal learning rate and weight decay for training with basic optimizers and keep them fixed for SAM and GAM.
We search $\rho$ in \{0.05, 0.1, 0.2, 0.5, 1.0, 2.0\} for both SAM and GAM and search $\alpha$ in \{0.1, 0.2, 0.5, 1.0, 2.0, 3.0, ..., 10.0\} for GAM. 
We set $\rho$ to 0.04 for CIFAR-10 and 0.1 for CIFAR-100. We set $\alpha$ to 0.3 for ResNet-18 and 0.1 for other models.
We report the best selection of hyperparameters for each individual model in Table \ref{app:tab:hp1} and Table \ref{app:tab:hp2}.

\subsubsection{Training Details of Transfer Learning Experiments}
We finetune the models on downstream datasets including Stanford Cars \cite{krause20133d}, CIFAR-10, CIFAR-100 \cite{krizhevsky2009learning}, Oxford\_IIIT\_Pets \cite{parkhi2012cats} and Food101 \cite{bossard2014food} from the weights pretrained on ImageNet. For EfficientNet-b0, we adopt SGD with momentum = 0.9. For Swin-t, we adopt AdamW with $\beta_1$ = 0.9, $\beta_2$ = 0.999. We train all the models for 40k steps and the batch size is 128. The initial learning rate is 2e-3 and the cosine learning rate decay is used. Weight decay is set to 1e-5. 

%% file: paras/appendix_accelerate.tex
\section{Further Acceleration of GAM}
\paragraph{Acceleration} Optimizing the gradient of $\GAM_{\rho}(\bftheta)$ according to Equations \eqref{eq:GAM-derivative} and \eqref{eq:gradient-norm} requires the Hessian vector product operation, which can still introduce considerable extra computation when the model is large. Inspired by \cite{singh2020woodfisher,zhao2022penalizing},
one can approximate $\nabla\|\nabla \hat{L}(\bftheta)\|$ with first-order gradient as follows.
\begin{equation} \label{eq:hessian-approx}
    \forall \bftheta \in \caltheta, \quad \nabla\left\|\nabla \hat{L}(\bftheta)\right\| \approx \frac{\nabla \hat{L}\left(\bftheta + \rho' \cdot \frac{\nabla \hat{L}(\bftheta)}{\|\nabla \hat{L}(\bftheta)\|}\right) - \nabla \hat{L}(\bftheta)}{\rho'},
\end{equation}
where $\rho'$ is a small constant. Thus we can further accelerate GAM by applying Equation \eqref{eq:hessian-approx} to the $\bftheta^{\text{adv}}$ term in Equation \eqref{eq:GAM-derivative} as follows,
\begin{equation}
    \bftheta^{\text{adv}} \approx \bftheta + \rho \cdot \frac{\nabla \hat{L}\left(\bftheta + \rho' \cdot \frac{\nabla \hat{L}(\bftheta)}{\|\nabla \hat{L}(\bftheta)\|}\right) - \nabla \hat{L}(\bftheta)}{\left\|\nabla \hat{L}\left(\bftheta + \rho' \cdot \frac{\nabla \hat{L}(\bftheta)}{\|\nabla \hat{L}(\bftheta)\|}\right) - \nabla \hat{L}(\bftheta)\right\|} = \bftheta + \rho \cdot \frac{\bfg_{1} - \bfg_{0}}{\left\|\bfg_{1} - \bfg_{0}\right\|},
\end{equation}
where
\begin{equation}
    \bfg_0 = \nabla \hat{L}(\bftheta), \quad \bfg_1 = \nabla \hat{L}(\tilde{\bftheta}_1), \quad \text{and} \quad \tilde{\bftheta}_1 = \bftheta + \rho' \cdot \frac{\nabla \hat{L}(\bftheta)}{\|\nabla \hat{L}(\bftheta)\|} = \bftheta + \rho' \cdot \frac{\bfg_0}{\|\bfg_0\|}. 
\end{equation}
We let $\tilde{\bftheta}_{2} \triangleq \bftheta^{\text{adv}}$. Now applying Equation \eqref{eq:hessian-approx} to the calculation of $\nabla \|\nabla \hat{L}(\tilde{\bftheta}_{2})\|$, we can get that
\begin{equation}
    \begin{aligned}
        \nabla \GAM_{\rho}(\bftheta) \approx \rho \cdot \nabla \left\|\nabla \hat{L}(\tilde{\bftheta}_{2})\right\| & \approx \rho \cdot \frac{\nabla \hat{L}\left(\tilde{\bftheta}_{2}  + \rho' \cdot \frac{\nabla \hat{L}(\tilde{\bftheta}_{2} )}{\|\nabla \hat{L}(\tilde{\bftheta}_{2} )\|}\right) - \nabla \hat{L}(\tilde{\bftheta}_{2})}{\rho'} \\
        & = \frac{\rho}{\rho'}\left(\bfg_{3} - \bfg_{2}\right),
    \end{aligned}
\end{equation}
where
\begin{equation}
    \bfg_{2} = \nabla \hat{L}(\tilde{\bftheta}_{2}),\quad \bfg_{3} = \nabla \hat{L}(\tilde{\bftheta}_{3}), \quad \text{and} \quad \tilde{\bftheta}_{3} = \tilde{\bftheta}_{2}  + \rho' \cdot \frac{\nabla \hat{L}(\tilde{\bftheta}_{2} )}{\|\nabla \hat{L}(\tilde{\bftheta}_{2} )\|} = \tilde{\bftheta}_{2}  + \rho' \cdot \frac{\bfg_{2}}{\|\bfg_{2}\|}.
\end{equation}

\paragraph{Accelerated GAM}
Based on the above approximations, we could obtain the accelerated version of GAM as shown in Algorithm \ref{alg:GAM_acceleration}. Besides them, the following modifications can further accelerate and suppress the effects of the above approximation.

\begin{algorithm}[t]
\caption{Accelerated GAM}
\label{alg:GAM_acceleration}
\begin{algorithmic}[1]
    \State \textbf{Input:} Batch size $b$, Learning rate $\eta_t$, Perturbation radius $\rho_t, \rho_t'$, Trade-off coefficient $\alpha, \beta, \gamma$, Small constant $\xi$
    \State $t \leftarrow 0$, $\bftheta_0 \leftarrow$ initial parameters
    \While{$\bftheta_t$ not converged}
        \State Sample $W_t$ from the training data with $b$ instances
        \State $\bfg_{t,0} \leftarrow \nabla \hat{L}_{W_t}(\bftheta_t)$
        \State $\tilde{\bftheta}_{t,1} \leftarrow \bftheta_t + \rho_t' \cdot \bfg_{t,0} / (\|\bfg_{t,0}\|+\xi)$
        \State $\bfg_{t,1} \leftarrow \nabla \hat{L}_{W_t}(\tilde{\bftheta}_{t,1})$
        \State $\bfh_{t,0} \leftarrow \bfg_{t,1} - \bfg_{t,0}$
        \State $\tilde{\bftheta}_{t,2} \leftarrow \bftheta_t + \rho_t \cdot \bfh_{t,0} / (\|\bfh_{t,0}\|+\xi)$
        \State $\bfg_{t,2} \leftarrow \nabla \hat{L}_{W_t}(\tilde{\bftheta}_{t,2})$
        \State $\tilde{\bftheta}_{t,3} \leftarrow \tilde{\bftheta}_{t,2} + \rho_t' \cdot \bfg_{t,2} / (\|\bfg_{t,2}\|+\xi)$
        \State $\bfg_{t,3} \leftarrow \nabla \hat{L}_{W_t}(\tilde{\bftheta}_{t,3})$
        \State $\bfh_{t, +} \leftarrow \alpha \bfg_{t,1} + (1 - \alpha)\bfg_{t,3}$
        \State $\bfh_{t, -} \leftarrow \beta \bfg_{t,0} + (1 - \beta)\bfg_{t,2}$
        \State $\bfh_{t, -}^{\parallel}, \bfh_{t, -}^{\perp} \leftarrow \mathrm{decompose}(\bfh_{t,-}; \bfh_{t, +})$ \Comment{Decompose $\bfh_{t,-}$ into components that are parallel or orthogonal to $\bfh_{t,+}$}
        \State $\bftheta_{t+1} \leftarrow \bftheta_{t} - \eta_t\left(\bfh_{t,+} - \gamma\bfh_{t, -}^{\perp}\right)$ 
        \State $t \leftarrow t + 1$
    \EndWhile
    \State \Return{$\bftheta_t$}
\end{algorithmic}
\end{algorithm}

\begin{enumerate}
    \item We find that the gradient of the SAM regularization term, $\bfg_{1} - \bfg_{0}$, is already calculated during the approximation steps. As a result, we directly optimize the target $\hat{L}(\bftheta) + \alpha'\GAM_{\rho}(\bftheta) + \beta' \sam_{\rho'}(\bftheta)$ with hyper-parameter $\alpha', \beta'$. The gradient of the target can be approximated as
    \begin{equation}
        \nabla \left(\hat{L}(\bftheta) + \alpha'\GAM_{\rho}(\bftheta) + \beta' \sam_{\rho'}(\bftheta)\right) \approx \bfg_0 + \frac{\alpha'\rho}{\rho'}(\bfg_3 - \bfg_2) + \beta'(\bfg_1 - \bfg_0) = \beta'\bfg_1 + \frac{\alpha'\rho}{\rho'}\bfg_3 - (\beta' - 1)\bfg_0 - \frac{\alpha'\rho}{\rho'}\bfg_2,
    \end{equation}
    which means that the gradient of our target is a linear combination of $\bfg_0, \bfg_1, \bfg_2, \bfg_3$. We further set three hyper-parameters to control the importance of these parts and preserve the signs of different parts, \textit{i.e.},
    \begin{equation} \label{eq:gradient-gam-acc-approximation}
        \nabla \left(\hat{L}(\bftheta) + \alpha'\GAM_{\rho}(\bftheta) + \beta' \sam_{\rho'}(\bftheta)\right) \approx \alpha \bfg_1 + (1 - \alpha)\bfg_3 - \gamma(\beta \bfg_0 + (1 - \beta) \bfg_2)
    \end{equation}
    with $1 \ge \alpha, \beta \ge 0, \gamma \ge 0$.
    \item We find that the negative parts of Equation \eqref{eq:gradient-gam-acc-approximation} may have side effects on the model's convergence and performance, and thus require fine-tuning of hyperparameters. 
    Inspired by \cite{xiong2022grod,zhuang2022surrogate}, we decompose $\beta \bfg_0 + (1 - \beta) \bfg_2$ into components that are parallel and orthogonal to $\alpha \bfg_1 + (1 - \alpha)\bfg_3$. Specifically, we first let
    \begin{equation}
        \bfh_+ = \alpha \bfg_1 + (1 - \alpha)\bfg_3, \quad \bfh_- = \beta \bfg_0 + (1 - \beta) \bfg_2
    \end{equation}
    to denote the positive and negative parts in Equation \eqref{eq:gradient-gam-acc-approximation}, respectively. We then decompose the negative part $\bfh_-$ into two components that are parallel or orthogonal to $\bfh_+$ and we get $\bfh_{-}^{\parallel}$ and $\bfh_{-}^{\perp}$. As a result, the final gradient is given by $\bfh_+ - \gamma\bfh_{-}^{\perp}$.
\end{enumerate}